\documentclass{article}
\usepackage{graphicx} % Required for inserting images
\usepackage{amsmath,amssymb,amsthm,tikz,float,booktabs}
\usepackage{cite}
\newtheorem{proposition}{Proposition}
\newtheorem{assumption}{Assumption}

\begin{document}

\begin{titlepage}
\centering

{\LARGE State-Coupled Volatility in Latent Dynamical Systems:
Recovery Under Partial Observation\par}

\vspace{1cm}

{\large
Imani Beckett\par
The Herbert Wertheim School of Public Health and Human Longevity Science\par
University of California San Diego\par
ibeckett@ucsd.edu\par
}

\vspace{0.5cm}

\end{titlepage}
\begin{abstract}

Latent state-space models are widely used to study partially observed dynamical systems, yet most formulations assume that process variability is independent of latent-state position. In many biological, behavioral, and physiological systems, however, variability may depend systematically on the underlying dynamical state, producing structured stochasticity that is not captured by constant-variance models.

We introduce a state-coupled stochastic volatility framework in which latent process variance depends on displacement from a latent equilibrium. To estimate this relationship under partial observation, we develop a particle expectation-maximization procedure combining bootstrap particle filtering and backward trajectory smoothing. The model includes a coupling parameter, $\gamma$, that quantifies the strength of association between latent-state position and process variability.

A large-scale simulation benchmark evaluated recovery and detection performance across varying coupling strengths, observation noise levels, trajectory lengths, and persistence regimes. The proposed framework consistently reduced recovery bias relative to an observed-state heteroskedastic proxy, with the largest improvements occurring under strong coupling. Recovery performance improved with increasing latent persistence, while detection performance remained competitive across a broad range of conditions and became increasingly advantageous as observation noise increased.

Taken together, the results demonstrate that state-coupled volatility can be identified and estimated under partial observation when latent-state structure is explicitly modeled. The framework provides a practical methodological foundation for studying state-dependent variability and evaluating whether structured stochasticity contributes information about system dynamics beyond that contained in mean-state trajectories alone.

\end{abstract}
\section*{Introduction}

Many biological, neural, and behavioral systems exhibit substantial variability across time, even under relatively stable experimental conditions. Traditional state-space and latent-variable approaches typically represent observations as the combination of an underlying latent process and residual noise,

\[
y_t=x_t+\epsilon_t,
\]

where \(x_t\) denotes the latent state and \(\epsilon_t\) denotes observation error or stochastic variation \cite{kalman1960new,durbin2012time}.

Although residual variation is often treated as nuisance noise, growing evidence suggests that variability may itself exhibit systematic organization associated with underlying system dynamics and information processing \cite{churchland2010stimulus,doiron2016mechanics,waschke2021behavior}. This organization may appear as temporal dependence, state-dependent fluctuation, coordinated residual structure, or regime-specific changes in variability \cite{dinstein2015neural,faisal2008noise}.

We refer to this phenomenon as \emph{state-dependent variability}: variability that contains statistical or dynamical organization beyond independent stochastic noise.

Related work has examined state-dependent variability through latent-variable, heteroskedastic, and uncertainty-aware modeling approaches. However, these approaches generally focus on state estimation, volatility forecasting, or representation learning rather than explicit recovery of latent state--variance relationships.

In this work, we introduce a variance-aware latent dynamical framework that treats residual variability as a potential dynamical feature rather than solely as measurement error. Conceptually,

\[
\epsilon_t=g(x_t,z_t,t)+\eta_t,
\]

where \(g(x_t,z_t,t)\) represents the state-dependent residual organization and \(\eta_t\) represents irreducible stochastic noise.

This formulation separates latent dynamics, state-dependent variability, and irreducible noise. Using simulation experiments, we evaluated whether state-dependent latent variability can be recovered and distinguished from conventional stochastic noise under partial observation. The current objective is parameter recoverability under the assumed model class rather than universal superiority over alternative volatility specifications.
\section*{Background and Related Literature}

The proposed model is based on established work in latent state-space modeling, stochastic volatility, conditional heteroskedasticity, and sequential Monte Carlo inference. Rather than introducing a wholly separate class of models, this work develops an interpretable heteroskedastic state-space formulation in which the transition variance is parameterized as a function of latent-state displacement from equilibrium. This formulation is motivated by state-dependent stochastic processes and modern nonlinear state-space models in which transition uncertainty may vary as a function of latent-state position or latent context \cite{krishnan2015deep,karl2016dvbf}. The main contribution is therefore an uncertainty-aware framework for estimating and calibrating state-dependent latent variability under partial observation.

\subsection*{Latent State-Space Models}

Latent state-space models provide a general framework for representing partially observed dynamical systems. In these models, an unobserved latent process evolves over time and generates noisy observations. Classical linear-Gaussian formulations, including Kalman filtering models, assume tractable transition and observation distributions, whereas nonlinear and non-Gaussian extensions allow greater flexibility through approximate inference methods \cite{kalman1960new,durbin2012time,cappe2005inference}.

Many standard state-space models assume that process noise has constant variance. This assumption is often useful computationally, but may be restrictive when the variability changes systematically with the state of the system. In biological, neural, behavioral, and physiological systems, fluctuations may not only reflect measurement error or unstructured noise; they may also exhibit statistical structure related to underlying system dynamics \cite{churchland2010stimulus,dinstein2015neural,doiron2016mechanics,waschke2021behavior}.

The present work adopts this perspective by retaining an interpretable latent mean dynamic while allowing transition variance to vary with latent displacement from equilibrium.

\subsection*{Stochastic Volatility and Conditional Heteroskedasticity}

Stochastic volatility models and GARCH-type models provide important foundations for modeling time-varying variance. In stochastic volatility models, variance is typically represented as an additional latent process \cite{jacquier1994bayesian,taylor1982financial}. In GARCH-type models, conditional variance depends on previous residuals and historical variance values \cite{engle1982autoregressive,bollerslev1986generalized}. These approaches have been widely used to describe volatility clustering and evolving uncertainty.

The present model is closely related to these traditions but targets a different inferential objective. The goal is not primarily to forecast volatility or recover an independent volatility trajectory. Instead, the objective is to estimate whether latent process variance changes systematically as the latent state moves away from equilibrium.

Specifically, transition variance is parameterized as

\begin{equation}
q_t = \exp\left(\alpha + \gamma |x_{t-1} - \mu|\right),
\end{equation}

where $\gamma$ measures the strength of association between latent displacement and process variability. This parameterization provides a direct interpretation of variance structure in terms of latent dynamical geometry. Also when $\gamma = 0$, the model reduces to a constant-variance latent dynamical system.

\subsection*{Heteroskedastic State-Space and Latent-Variable Models}

The proposed framework may be viewed as a restricted heteroskedastic state-space model with state-dependent transition variance. Heteroskedastic regression and latent-variable models commonly allow residual or process variance to depend on observed covariates \cite{harvey1976estimating}. In the present formulation, however, the variance-driving quantity,

\begin{equation}
z_t = |x_{t-1}-\mu|,
\end{equation}

is not directly observed and must instead be inferred from noisy observations of the latent process.

This creates an estimation challenge because uncertainty in latent-state reconstruction propagates directly into uncertainty about the variance function. A plug-in approach that estimates a single latent trajectory and subsequently fits a variance relationship may overstate confidence or distort coupling estimates. For this reason, the present framework emphasizes trajectory-level uncertainty rather than treating latent-state reconstruction as a completed preprocessing step. This distinction is important because the quantity governing variance is itself latent. Consequently, estimation of the variance function depends on accurate reconstruction of the underlying state trajectory, creating a coupling between state estimation and variance estimation that is absent in conventional heteroskedastic regression models.

\subsection*{Particle Filtering and Smoothing}

Sequential Monte Carlo methods provide flexible tools for inference in nonlinear and non-Gaussian state-space models \cite{gordon1993novel,doucet2001sequential}. Particle filtering approximates filtering distributions using weighted samples, while particle smoothing enables approximation of posterior distributions over complete latent trajectories.

In this work, particle methods are used not merely to estimate latent trajectories but to propagate latent-state uncertainty into estimation of the state-dependent variance relationship. This distinction is important because the same latent states that define system trajectories also determine process variance through the coupling function. The inference problem is therefore jointly about recovering latent dynamics and evaluating whether variability exhibits systematic organization around those dynamics.

Sequential Monte Carlo methods have continued to develop substantially beyond classical bootstrap particle filtering. Modern approaches include particle Markov chain Monte Carlo (PMCMC) methods for joint state and parameter inference \cite{andrieu2010particle}, particle Gibbs samplers with ancestor sampling to mitigate path degeneracy \cite{lindsten2014particle}, SMC$^2$ algorithms that perform sequential inference over both latent states and parameters \cite{chopin2013smc2}, and variational sequential Monte Carlo methods that combine particle-based inference with variational optimization \cite{naesseth2018variational}. These developments have improved parameter uncertainty quantification, computational efficiency, and scalability in increasingly complex state-space models. The present work adopts a bootstrap particle filter with backward-sampling smoothing because it provides a transparent and computationally straightforward estimation procedure sufficient for evaluating recoverability of the proposed state-coupled volatility framework under controlled simulation conditions.

\subsection*{Positioning of the Present Contribution}

The contribution of this work is not that state-dependent variance models are entirely new. Rather, the contribution is a focused and interpretable formulation for studying state-dependent variability in partially observed latent dynamical systems.

The framework combines three components:

\begin{enumerate}
\item an equilibrium-seeking latent dynamical model;
\item a state-coupled transition variance parameterization;
\item particle smoothing and particle EM estimation to propagate latent uncertainty;
\end{enumerate}

Together, these components provide a practical approach for evaluating whether observed variability reflects state-independent stochastic fluctuation or state-dependent latent volatility. 
\section*{State-Coupled Stochastic Volatility Framework}

\subsection*{Latent Dynamical System}

Consider a latent state process $\{x_t\}_{t=1}^{T}$ governed by

\[
x_t
=
\mu
+
\phi(x_{t-1}-\mu)
+
w_t,
\]

where:
\begin{itemize}
    \item $\mu \in \mathbb{R}$ denotes the equilibrium or long-run mean,
    \item $\phi \in (-1,1)$ controls autoregressive persistence,
    \item $w_t$ denotes process noise.
\end{itemize}

Observed measurements are generated through

\[
y_t = x_t + v_t,
\]

with observation noise

\[
v_t \sim \mathcal{N}(0,r),
\]

where $r > 0$ denotes observation variance. \\
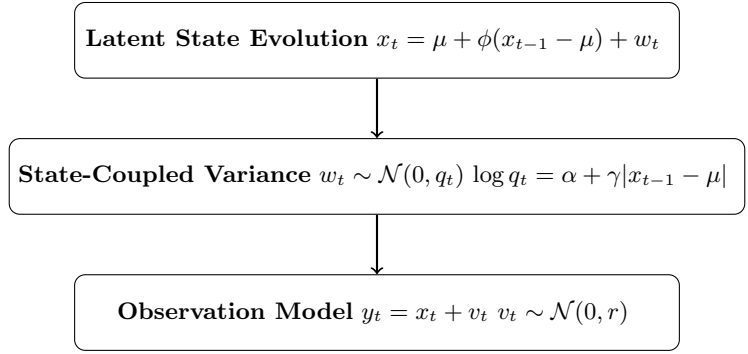
\begin{figure}[ht]
\centering

\begin{tikzpicture}[
    node distance=1.8cm,
    every node/.style={font=\small},
    box/.style={
        draw,
        rounded corners,
        minimum width=8cm,
        minimum height=1cm,
        align=center
    },
    arrow/.style={
        ->,
        thick
    }
]

% Nodes
\node[box] (latent)
{
\textbf{Latent State Evolution}

\vspace{0.2cm}

$ x_t = \mu + \phi(x_{t-1}-\mu) + w_t $
};

\node[box, below of=latent] (variance)
{
\textbf{State-Coupled Variance}

\vspace{0.2cm}

$ w_t \sim \mathcal{N}(0,q_t) $

\vspace{0.1cm}

$ \log q_t = \alpha + \gamma |x_{t-1}-\mu| $
};

\node[box, below of=variance] (obs)
{
\textbf{Observation Model}

\vspace{0.2cm}

$ y_t = x_t + v_t $

\vspace{0.1cm}

$ v_t \sim \mathcal{N}(0,r) $
};

% Arrows
\draw[arrow] (latent) -- (variance);
\draw[arrow] (variance) -- (obs);

\end{tikzpicture}

\caption{
Overview of the state-coupled stochastic volatility framework.
The latent process variance depends on displacement from equilibrium through the coupling parameter $\gamma$.
}

\label{fig:scsv_model}

\end{figure}
\subsection*{State-Coupled Process Variance}

Unlike standard latent state-space models assuming constant process variance, we allow the latent process variance to depend on the current latent state.

Specifically,
\[
\log q_t = \alpha + \gamma |x_{t-1}-\mu|
\]
where:
\begin{itemize}
    \item $q_t$ denotes the latent process variance at time $t$,
    \item $\alpha$ controls baseline volatility,
    \item $\gamma$ controls state-coupled volatility strength.
\end{itemize}

The latent innovation process is therefore

\[
w_t \sim \mathcal{N}(0,q_t).
\]

Combining these equations yields the complete latent evolution model:

\[
x_t
=
\mu
+
\phi(x_{t-1}-\mu)
+
\varepsilon_t,
\]

with

\[
\varepsilon_t \sim \mathcal{N}(0,q_t),
\]

and

\[
q_t
=
\exp\left(
\alpha
+
\gamma |x_{t-1}-\mu|
\right).
\]

\subsection*{Interpretation of the Coupling Parameter}

The parameter $\gamma$ determines how latent-state displacement modulates process variability.

\begin{itemize}
    \item $\gamma = 0$ reduces the model to a constant-variance latent autoregressive process.
    \item $\gamma > 0$ implies increased latent variability as the latent state moves farther from equilibrium.
    \item Larger values of $\gamma$ correspond to stronger state-dependent volatility coupling.
\end{itemize}

Thus, the model permits variability structure to dynamically change as a function of latent-state geometry.
\subsection*{Stochastic Assumptions and Innovation Process}
Let $\{x_t\}_{t=1}^T$ denote the latent state process and let
$\{y_t\}_{t=1}^T$ denote the observed process. The model is defined by the
conditional transition equation
\[
x_t \mid x_{t-1} \sim
\mathcal{N}\left(
\mu + \phi(x_{t-1}-\mu), q_t
\right),
\]
where the process variance is state-coupled:
\[
q_t = \exp\left(\alpha + \gamma |x_{t-1}-\mu|\right).
\]

Equivalently, the latent process may be written as
\[
x_t = \mu + \phi(x_{t-1}-\mu) + \epsilon_t,
\]
where the innovation process satisfies
\[
\epsilon_t \mid x_{t-1}
\sim \mathcal{N}(0,q_t).
\]

The observation equation is
\[
y_t = x_t + v_t,
\]
with
\[
v_t \sim \mathcal{N}(0,r),
\qquad r>0.
\]

We assume that, conditional on the latent trajectory, the observation errors
$\{v_t\}$ are independent across time and independent of the latent innovations
$\{\epsilon_t\}$.
\subsection*{Model Assumptions}

The proposed state-coupled stochastic volatility model is studied under the
following assumptions.

\begin{assumption}[Stability]
The autoregressive coefficient satisfies
\[
|\phi| < 1.
\]
This ensures that, in the absence of state-coupled volatility, the latent process
has a stable equilibrium at $\mu$.
\end{assumption}

\begin{assumption}[Positive process variance]
For all $t$,
\[
q_t = \exp\left(\alpha + \gamma |x_{t-1}-\mu|\right) > 0.
\]
Thus the process variance is strictly positive for all latent states.
\end{assumption}

\begin{assumption}[Observation noise]
The observation errors satisfy
\[
v_t \overset{\text{iid}}{\sim} \mathcal{N}(0,r),
\qquad r>0,
\]
and are independent of the latent innovations conditional on the latent state.
\end{assumption}

\begin{assumption}[Conditional innovation structure]
The latent innovations satisfy
\[
\epsilon_t \mid x_{t-1}
\sim \mathcal{N}
\left(
0,
\exp\left(\alpha+\gamma |x_{t-1}-\mu|\right)
\right).
\]
Thus the innovations are conditionally mean-zero but not identically distributed
when $\gamma \neq 0$.
\end{assumption}

\begin{assumption}[Finite parameter space]
The parameters are restricted to compact sets:
\[
\mu \in \mathcal{M}, \qquad
\phi \in (-1,1), \qquad
\alpha \in \mathcal{A}, \qquad
\gamma \in \mathcal{G}, \qquad
r \in \mathcal{R}\subset (0,\infty).
\]
This avoids degenerate variance estimates and supports stable numerical
estimation.
\end{assumption}
\subsection*{Particle EM Estimation}

Inference was implemented using bootstrap particle filtering with backward-sampling smoothing \cite{gordon1993novel}.

At iteration $k$, particles were propagated according to

\[
x_t^{(i)}
\sim
\mathcal N
\left(
\mu+\phi(x_{t-1}^{(i)}-\mu),
q_t^{(i)}
\right),
\]

with

\[
q_t^{(i)}
=
\exp(\alpha+\gamma|x_{t-1}^{(i)}-\mu|).
\]

Particle weights were updated using

\[
y_t\mid x_t^{(i)}
\sim
\mathcal N(x_t^{(i)},r),
\]

followed by systematic resampling.

Backward trajectory sampling generated

\[
x_{1:T}^{(1)},\ldots,x_{1:T}^{(M)}
\sim
p(x_{1:T}\mid y_{1:T}).
\]

\subsubsection*{EM Updates}
\begin{figure} [H]
    \centering
    \includegraphics[width=1\linewidth]{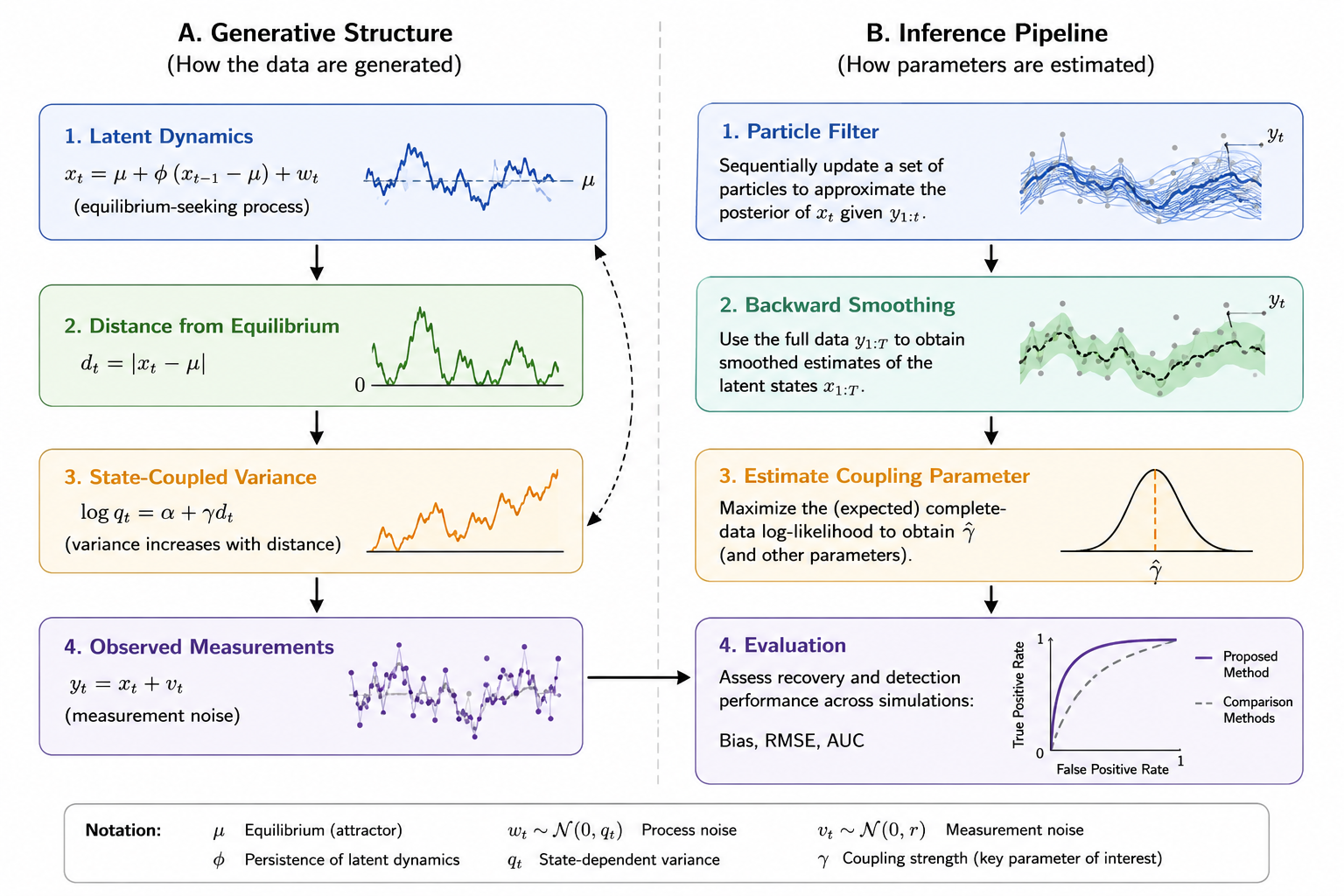}
    \caption{(A) Generative structure. Latent dynamics evolve according to an equilibrium-seeking process, where the latent state $x_t$ fluctuates around an equilibrium $\mu$ with persistence parameter $\phi$. The magnitude of displacement from equilibrium determines the process variance through a state-coupling mechanism, $\log q_t = \alpha + \gamma |x_t-\mu|$, where $\gamma$ quantifies the strength of state-dependent variability. Observed measurements $y_t$ are generated from the latent state under measurement noise. (B) Inference pipeline. Particle filtering and backward smoothing are used to reconstruct latent trajectories from noisy observations. Model parameters, including the coupling parameter $\gamma$, are estimated using a particle expectation-maximization procedure. Recovery and detection performance are then evaluated using simulation benchmarks, including bias, root mean squared error (RMSE), and area under the receiver operating characteristic curve (AUC). The central innovation of the framework is the explicit modeling of variability as a function of latent-state position, allowing structured state-dependent stochasticity to be recovered under partial observation}
    \label{fig:framgraph}
\end{figure}
The E-step approximated the smoothing distribution using sampled trajectories.
In the present simulation benchmark, $\mu$, $\phi$, and r were treated as fixed inputs, and the particle EM procedure estimated only the volatility parameters $\alpha$ and $\gamma$. EM updates were damped by averaging the previous estimate with the current M-step estimate.
The M-step updated volatility parameters

\[
\theta_v=(\alpha,\gamma),
\]

while holding $\mu$, $\phi$, and $r$ fixed.

Residuals and coupling covariates were computed as

\[
e_t^{(m)}
=
x_t^{(m)}
-
\left[
\mu+\phi(x_{t-1}^{(m)}-\mu)
\right],
\]

\[
z_t^{(m)}
=
|x_{t-1}^{(m)}-\mu|.
\]

Parameters were estimated by minimizing

\[
\mathcal L(\alpha,\gamma)
=
\frac12
\sum_{m=1}^{M}
\sum_{t=2}^{T}
\left[
\log q_t^{(m)}
+
\frac{(e_t^{(m)})^2}{q_t^{(m)}}
\right],
\]

where

\[
q_t^{(m)}
=
\exp(\alpha+\gamma z_t^{(m)}).
\]

Optimization used bounded L-BFGS-B:

\[
\alpha\in[-8,4],
\qquad
\gamma\in[-5,5].
\]

\subsubsection*{Implementation}
Monte Carlo EM updates were damped:

\[
\theta_v^{(k+1)}
=
(1-\lambda)\theta_v^{(k)}
+
\lambda\tilde\theta_v^{(k+1)},
\]

with

\[
\lambda=0.5.
\]

Iterations terminated when parameter changes were below

\[
10^{-3},
\]

or after 20 iterations.

Unless otherwise stated, simulations used
$N=3000$ particles and
$M=300$ backward-sampled trajectories (Figure~\ref{fig:framgraph}).
\subsection*{Identifiability Under Observed Latent States}

The following proposition establishes identifiability of the volatility parameters when the latent trajectory is known and the mean-dynamics parameters are fixed. \\
\begin{proposition}[Conditional identifiability]

Consider the latent transition model

\[
x_t
=
\mu+\phi(x_{t-1}-\mu)+\varepsilon_t,
\]

where

\[
\varepsilon_t\mid x_{t-1}
\sim
N(0,q_t),
\]

and

\[
q_t
=
\exp
\{
\alpha+\gamma|x_{t-1}-\mu|
\}.
\]

Define

\[
z_t
=
|x_{t-1}-\mu|.
\]

Assume that $\mu$ and $\phi$ are known, the latent trajectory $x_{1:T}$ is observed, and $z_t$ takes at least two distinct values with positive probability.

Then the volatility parameters $(\alpha,\gamma)$ are identifiable.

\end{proposition}

\begin{proof}

Because $x_{1:T}$ is observed and $\mu,\phi$ are known,

\[
e_t
=
x_t-
[
\mu+\phi(x_{t-1}-\mu)
]
\]

is uniquely determined.

Equality of Gaussian conditional distributions implies

\[
\exp(\alpha+\gamma z)
=
\exp(\alpha'+\gamma' z).
\]

Taking logarithms yields

\[
(\alpha-\alpha')
+
(\gamma-\gamma')z
=
0.
\]

Since $z_t$ takes at least two distinct values,

\[
\gamma=\gamma',
\]

and therefore

\[
\alpha=\alpha'.
\]

Hence $(\alpha,\gamma)$ are identifiable.

\end{proof}

Proposition 1 establishes structural identifiability of the volatility
mapping conditional on knowledge of the latent trajectory and the
mean-dynamics parameters. This result should not be interpreted as a
proof of identifiability for the fully observed-data model, since under
partial observation the latent states themselves must be estimated.
Consequently, the practical recoverability of $(\alpha,\gamma)$ depends
on observational noise, trajectory length, latent persistence, and the
accuracy of the particle-based approximation. These factors are
investigated empirically in the simulation study.
\section*{Simulation Experiments}

\subsection*{Simulation Setup}

Simulation experiments evaluated recovery of state-coupled volatility under partial observation. Synthetic datasets were generated from

\[
x_t
=
\mu+\phi(x_{t-1}-\mu)+\varepsilon_t,
\]

with

\[
\varepsilon_t \sim \mathcal N(0,q_t),
\]

where

\[
q_t
=
\exp\left(\alpha+\gamma |x_{t-1}-\mu|\right).
\]

Observed measurements were generated as

\[
y_t=x_t+v_t,
\qquad
v_t\sim\mathcal N(0,r).
\]

Unless otherwise stated, simulations used

\[
\mu=0,
\qquad
\alpha=-2.5.
\]

Simulation conditions varied coupling strength, latent persistence, observation noise, and trajectory length (Table~\ref{tab:simulation_grid}). For each parameter combination, 30 independent replicates were generated using randomized seeds.

\begin{table}[ht]
\centering
\caption{Simulation parameter grid used in the benchmark study.}
\label{tab:simulation_grid}
\begin{tabular}{ll}
\toprule
$\gamma$ (coupling strength) &
$\{0.0,\,0.3,\,0.6,\,0.8,\,1.0\}$ \\

$\phi$ (latent persistence) &
$\{0.3,\,0.6,\,0.9\}$ \\

$r$ (observation noise) &
$\{0.01,\,0.05,\,0.10\}$ \\

$T$ (trajectory length) &
$\{250,\,500,\,1000\}$ \\
\bottomrule
\end{tabular}
\end{table}

The full benchmark therefore consisted of

\[
5\times3\times3\times3\times30
=
4050
\]

independent simulated datasets.

\subsection*{Evaluation Metrics}

Performance was evaluated using three complementary criteria:

\begin{itemize}
    \item parameter recovery: bias and RMSE;
    \item coupling discrimination: ROC AUC across coupling thresholds;
\end{itemize}

For each simulated dataset, the complete inference pipeline was applied, including particle filtering, backward trajectory smoothing, particle EM estimation, and empirical null calibration. 
\subsubsection*{Coupling Discrimination and AUC}

Detection performance was evaluated using receiver operating characteristic (ROC) analysis. For each simulated dataset, the estimated coupling parameter $\hat{\gamma}$ served as the detection statistic. Binary class labels were defined using the true coupling parameter. Datasets generated with $\gamma_{\text{true}} \ge 0.3$ were classified as positive examples, whereas datasets generated with $\gamma_{\text{true}} < 0.3$ were classified as negative examples. Given the simulation design $\gamma \in \{0.0, 0.3, 0.6, 0.8, 1.0\}$, this corresponds to treating all nonzero coupling conditions as positives and the uncoupled condition ($\gamma = 0$) as the negative class.

For each combination of model, latent persistence ($\phi$), observation noise ($r$), and trajectory length ($T$), ROC curves were constructed by varying a decision threshold over the range of estimated coupling values $\hat{\gamma}$. At each threshold, the corresponding true-positive and false-positive rates were computed, and the area under the ROC curve (AUC) was calculated using standard trapezoidal integration. An AUC value of $0.5$ indicates chance-level discrimination, whereas an AUC value of $1.0$ indicates perfect separation between coupled and uncoupled systems.

This analysis evaluates the ability of each method to detect the presence of state-coupled volatility, irrespective of the exact magnitude of the underlying coupling. The threshold $\gamma_{\text{true}} = 0.3$ was selected because it represents the smallest nonzero coupling value included in the simulation design, thereby operationalizing detection of any departure from the null condition.
\subsubsection*{Confidence Intervals}

To characterize uncertainty in simulation-based performance estimates, 95$\%$ confidence intervals were computed for recovery and detection metrics across simulation replicates.

For parameter recovery, bias was defined as

\begin{equation}
\text{Bias} = \hat{\gamma} - \gamma,
\end{equation}

where $\hat{\gamma}$ denotes the estimated coupling parameter and $\gamma$ the true data-generating value. Confidence intervals for mean bias were computed using the Student-$t$ distribution,

\begin{equation}
\bar{b}
\pm
t_{0.975,n-1}
\frac{s_b}{\sqrt{n}},
\end{equation}

where $\bar{b}$ is the mean bias across simulation replicates, $s_b$ is the sample standard deviation of the bias estimates, and $n$ is the number of replicates.

For root mean squared error (RMSE), confidence intervals were obtained using nonparametric bootstrap resampling. Within each simulation condition, replicate-level estimation errors were resampled with replacement 2,000 times. RMSE was recomputed for each bootstrap sample, and the 2.5th and 97.5th percentiles of the resulting bootstrap distribution were used to form a 95\% confidence interval.
\subsection*{Benchmark Estimators}

Performance was compared against four benchmark procedures representing distinct approaches to modeling variability in partially observed dynamical systems. All benchmark models employed the same latent-state dynamics,

\begin{equation}
x_t
=
\mu
+
\phi(x_{t-1}-\mu)
+
\epsilon_t,
\end{equation}

but differed in how process variance was specified. This design isolates the contribution of the variance model while holding the underlying latent dynamics constant.

\paragraph{Constant-Variance Latent Dynamical System}

The Constant LDS assumes homoscedastic process noise:

\begin{equation}
\epsilon_t \sim \mathcal{N}(0,q),
\end{equation}

where \(q\) is a constant variance parameter estimated from the data.

\paragraph{Observed-State Heteroskedastic Proxy}

Variance was modeled as a function of observed displacement from equilibrium:

\begin{equation}
\epsilon_t \sim \mathcal{N}(0,q_t),
\end{equation}

\begin{equation}
\log q_t
=
\alpha
+
\beta \left| y_{t-1}-\mu \right|.
\end{equation}

This benchmark assesses whether observed-state variability alone is sufficient to recover state-dependent variance structure.

\paragraph{GARCH Proxy}

Conditional variance followed a GARCH(1,1) recursion:

\begin{equation}
\epsilon_t \sim \mathcal{N}(0,h_t),
\end{equation}

\begin{equation}
h_t
=
\omega
+
a\epsilon_{t-1}^{2}
+
bh_{t-1}.
\end{equation}

This model captures volatility clustering through past residual variance but does not explicitly incorporate latent-state position.

\paragraph{Stochastic Volatility Proxy}

As a comparison to state-dependent variance models, a simplified stochastic volatility proxy was constructed in which volatility evolves independently of latent-state position. Let

\begin{equation}
e_t = y_t - \hat{y}_t
\end{equation}

denote prediction residuals obtained from the latent-state model. Log-variance dynamics were then approximated by

\begin{equation}
z_t
=
\rho z_{t-1}
+
\eta_t,
\end{equation}

\begin{equation}
\eta_t
\sim
\mathcal{N}(0,\sigma_z^2),
\end{equation}

with implied variance

\begin{equation}
q_t = \exp(z_t).
\end{equation}

Unlike the proposed framework, volatility evolves independently of latent-state position and therefore cannot represent state-dependent variance structure. This benchmark is intended to capture generic time-varying volatility rather than provide a full stochastic-volatility state-space implementation.

\subsection*{Particle EM Configuration}

The proposed estimator used sequential Monte Carlo inference combined with particle expectation maximization. Default estimation settings are summarized in Table~\ref{tab:pem_config}.

\begin{table}[ht]
\centering
\caption{Particle EM estimation configuration.}
\label{tab:pem_config}
\begin{tabular}{ll}
\toprule
Component & Setting \\
\midrule
Initial $\alpha$ & $-2.0$ \\
Initial $\gamma$ & $0.3$ \\
Particle count& 3000 \\
Backward sampled trajectories & 300 \\
Maximum EM iterations & 20 \\
Particle resampling & Systematic \\
Optimization & L-BFGS-B \\
$\alpha$ bounds & $[-8,4]$ \\
$\gamma$ bounds & $[-5,5]$ \\
Convergence tolerance & $10^{-3}$ \\
Damping parameter & 0.5 \\
\bottomrule
\end{tabular}
\end{table}
Table~\ref{tab:pem_config} summarizes the default particle EM settings used throughout the simulation study. The estimation procedure was initialized at $\alpha=-2.0$ and $\gamma=0.3$, providing reasonable starting values near the range of simulated parameters while avoiding initialization at the true generating values. Filtering was performed using 3000 particles, and posterior trajectory uncertainty was approximated using 300 backward-sampled trajectories from the smoothing distribution. Parameter updates were obtained by minimizing the volatility objective using the bounded L-BFGS-B algorithm, with $\alpha \in [-8,4]$ and $\gamma \in [-5,5]$ to prevent numerically unstable variance estimates. EM updates were damped using a learning rate of $\lambda=0.5$ to reduce Monte Carlo variability and improve convergence stability. Iterations terminated when successive parameter estimates changed by less than $10^{-3}$ or after a maximum of 20 iterations.
The volatility parameters $(\alpha,\gamma)$ were estimated using the bounded L-BFGS-B algorithm. This optimizer was selected because the M-step objective is smooth and continuously differentiable with respect to both parameters, making gradient-based optimization efficient. In addition, L-BFGS-B permits explicit parameter bounds, allowing the search to be restricted to numerically stable regions of the parameter space and preventing extreme variance estimates that can arise from the exponential volatility mapping. Because only two parameters are optimized during each M-step, L-BFGS-B provides rapid convergence while maintaining robustness across simulation conditions \cite{zhu1997lbfgsb}.
\subsection*{Results}
\subsubsection*{Recovery Across Coupling Strength and Observation Noise}
Recovery performance exhibited a clear bias--variance tradeoff. The proposed state-coupled framework consistently produced substantially lower bias than the observed-state heteroskedastic proxy across all coupling strengths (Table ~\ref{tab:proxy_recovery}, Table ~\ref{tab:state_recovery}). For example, when $\gamma=1.0$, the heteroskedastic proxy underestimated coupling by approximately $0.42$ on average, whereas the proposed method underestimated coupling by only $0.11$ (Figure ~\ref{fig:gamabs}). This pattern persisted across all nonzero coupling regimes, indicating that the proposed framework more accurately recovered the underlying coupling parameter.

Despite improved bias, the proposed model exhibited higher RMSE in several regimes, reflecting a tradeoff between systematic error and estimation variability. This suggests that the reduction in systematic error was accompanied by increased estimation variability. Notably however, the RMSE gap decreased substantially as coupling strength increased. At $\gamma=0.3$, the difference in RMSE between the two approaches was approximately $0.25$, whereas at $\gamma=1.0$ the difference was reduced to approximately $0.03$ (Figure~\ref{fig:gamcoup}). These findings suggest that the proposed framework becomes increasingly competitive as state-dependent variability becomes stronger and therefore easier to identify from noisy observations.
\begin{figure} [H]
    \centering
    \includegraphics[width=1\linewidth]{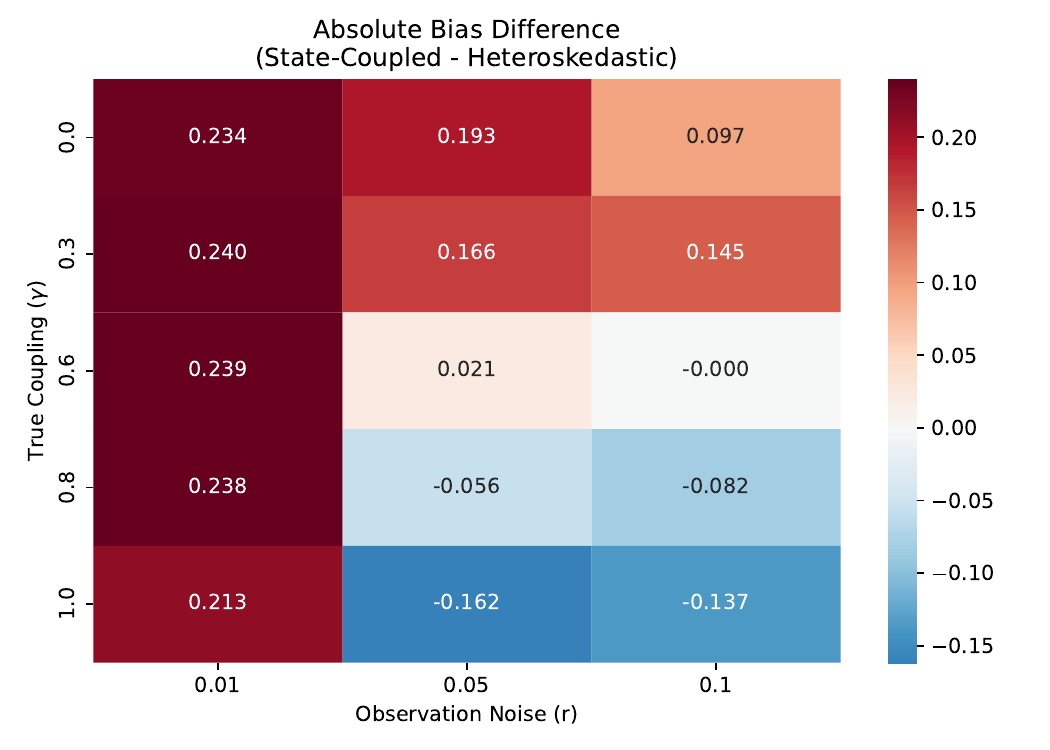}
    \caption{
Difference in absolute recovery bias between the state-coupled framework and the observed-state heteroskedastic proxy across coupling strength ($\gamma$) and observation noise ($r$). Values represent
$|\mathrm{Bias}|_{\text{Proxy}} - |\mathrm{Bias}|_{\text{State-Coupled}}$,
with positive values indicating lower absolute bias for the proposed model. Bias reduction was greatest under weak-to-moderate observation noise and remained substantial across a broad range of coupling strengths.
}
    \label{fig:gamabs}
\end{figure}
\begin{figure}[H]
    \centering
    \includegraphics[width=1\linewidth]{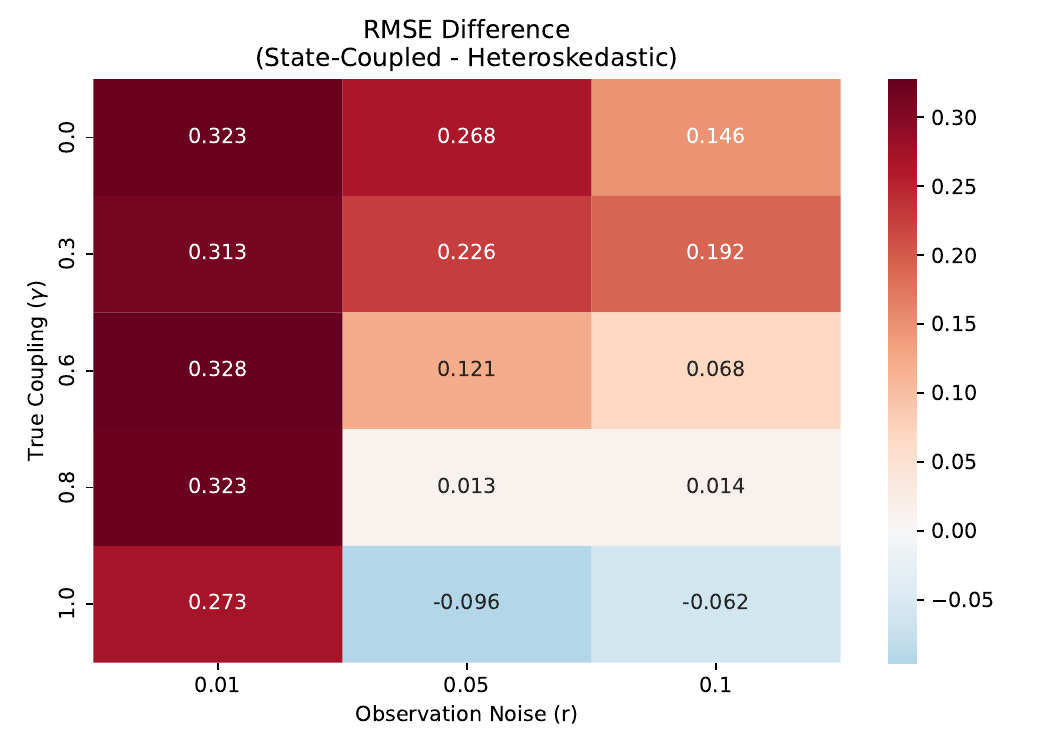}
    \caption{
Difference in parameter recovery RMSE between the state-coupled framework and the observed-state heteroskedastic proxy across coupling strength ($\gamma$) and observation noise ($r$). Values represent
$\mathrm{RMSE}_{\text{Proxy}} - \mathrm{RMSE}_{\text{State-Coupled}}$,
with positive values indicating lower recovery error for the proposed model. Recovery improvements were largest under low observation noise and generally increased with coupling strength.
}
    \label{fig:gamcoup}
\end{figure}

\begin{table}[ht]
\centering
\caption{Recovery performance of the heteroskedastic proxy. Each cell reports Bias (95$\%$ CI) / RMSE (95$\%$ CI).}
\label{tab:proxy_recovery}
\small
\begin{tabular}{lccc}
\toprule
$\gamma$ & $r=0.01$ & $r=0.05$ & $r=0.10$ \\
\midrule

0.0 &
\shortstack{-0.040 [-0.076,-0.004] \\ 0.301 [0.266,0.335]} &
\shortstack{0.089 [0.058,0.120] \\ 0.273 [0.247,0.299]} &
\shortstack{0.169 [0.141,0.197] \\ 0.289 [0.265,0.312]} \\

0.3 &
\shortstack{-0.084 [-0.115,-0.053] \\ 0.272 [0.243,0.304]} &
\shortstack{-0.123 [-0.152,-0.094] \\ 0.271 [0.245,0.298]} &
\shortstack{-0.102 [-0.131,-0.073] \\ 0.259 [0.233,0.285]} \\

0.6 &
\shortstack{-0.128 [-0.158,-0.097] \\ 0.284 [0.246,0.330]} &
\shortstack{-0.299 [-0.330,-0.269] \\ 0.391 [0.356,0.427]} &
\shortstack{-0.322 [-0.349,-0.295] \\ 0.393 [0.367,0.419]} \\

0.8 &
\shortstack{-0.151 [-0.183,-0.118] \\ 0.309 [0.275,0.343]} &
\shortstack{-0.377 [-0.408,-0.346] \\ 0.457 [0.424,0.493]} &
\shortstack{-0.457 [-0.481,-0.434] \\ 0.498 [0.473,0.523]} \\

1.0 &
\shortstack{-0.194 [-0.226,-0.162] \\ 0.331 [0.297,0.370]} &
\shortstack{-0.481 [-0.510,-0.453] \\ 0.538 [0.509,0.566]} &
\shortstack{-0.580 [-0.607,-0.552] \\ 0.624 [0.595,0.652]} \\

\bottomrule
\end{tabular}
\end{table}
\begin{table}[ht]
\centering
\caption{Recovery performance of the state-coupled model. Each cell reports Bias (95$\%$ CI) / RMSE (95$\%$ CI).}
\label{tab:state_recovery}
\small
\begin{tabular}{lccc}
\toprule
$\gamma$ & $r=0.01$ & $r=0.05$ & $r=0.10$ \\
\midrule

0.0 &
\shortstack{-0.013 [-0.088,0.062] \\ 0.624 [0.553,0.690]} &
\shortstack{-0.042 [-0.106,0.023] \\ 0.540 [0.488,0.594]} &
\shortstack{0.048 [-0.003,0.100] \\ 0.435 [0.392,0.486]} \\

0.3 &
\shortstack{0.136 [0.067,0.204] \\ 0.585 [0.529,0.645]} &
\shortstack{-0.114 [-0.172,-0.056] \\ 0.498 [0.448,0.548]} &
\shortstack{-0.121 [-0.173,-0.069] \\ 0.451 [0.405,0.493]} \\

0.6 &
\shortstack{0.211 [0.142,0.280] \\ 0.611 [0.531,0.694]} &
\shortstack{-0.112 [-0.172,-0.052] \\ 0.512 [0.436,0.594]} &
\shortstack{-0.227 [-0.275,-0.179] \\ 0.460 [0.411,0.509]} \\

0.8 &
\shortstack{0.301 [0.234,0.367] \\ 0.632 [0.565,0.696]} &
\shortstack{-0.120 [-0.174,-0.065] \\ 0.471 [0.416,0.526]} &
\shortstack{-0.339 [-0.385,-0.293] \\ 0.512 [0.451,0.572]} \\

1.0 &
\shortstack{0.263 [0.198,0.328] \\ 0.604 [0.547,0.662]} &
\shortstack{-0.177 [-0.226,-0.129] \\ 0.444 [0.396,0.492]} &
\shortstack{-0.419 [-0.464,-0.375] \\ 0.562 [0.504,0.618]} \\

\bottomrule
\end{tabular}
\end{table}
\subsubsection*{Recovery Across Trajectory Length and Observation Noise}
Recovery behavior depended strongly on both trajectory length and observation noise. A consistent pattern across all simulation settings was the tendency of the heteroskedastic proxy to underestimate the true coupling parameter. This attenuation became increasingly severe as measurement noise increased. For shorter trajectories $T=250$, mean bias shifted from approximately $-0.15$ at $r=0.01$ to nearly $-0.29$ at $r=0.10$, suggesting that direct reliance on noisy observations increasingly obscures the underlying state-dependent structure.

The proposed state-coupled framework exhibited a markedly different pattern. Although bias was not eliminated entirely, estimates remained substantially closer to the true coupling parameter across most moderate- and high-noise regimes. For example, at $T=1000$ and $r=0.05$, the proposed model produced a bias of $-0.088$ , compared with $-0.237$ for the heteroskedastic proxy (Figure~\ref{fig:trajbias}). This reduction in attenuation was observed repeatedly across the simulation grid, indicating that latent-state reconstruction helps preserve information about the underlying coupling mechanism even when observations are corrupted by measurement noise .

The improvement in bias came at the cost of increased variability. In low-noise settings, where observations already provided a relatively faithful representation of the latent process, the heteroskedastic proxy achieved substantially lower RMSE. At $T=1000$ and $r=0.01$, for instance, the proxy achieved an RMSE of $0.209$), compared with $0.480$ for the proposed method. Thus, while the latent model more accurately recovered the coupling parameter on average, it did so with greater finite-sample uncertainty.

Notably, this RMSE disadvantage diminished as additional temporal information became available. Across all noise levels, recovery error for the proposed model decreased steadily with increasing trajectory length. At $r=0.10$, RMSE declined from $0.595$ at $T=250$ to $0.386$ at $T=1000$, while corresponding reductions were observed in the lower-noise settings as well (Table ~\ref{tab:recovery_ci_r_T}). This suggests that a substantial portion of the observed error reflects estimation uncertainty rather than structural deficiencies of the model itself.

Taken together, these results point to a clear tradeoff. When observations are highly informative, simpler observed-state approaches provide stable estimates with relatively low error. As measurements become noisier, however, those methods increasingly underestimate the strength of state-dependent variability. The proposed state-coupled framework sacrifices some finite-sample efficiency in exchange for substantially improved recovery fidelity, and becomes progressively more competitive as both observation noise and trajectory length increase (Table~\ref{fig:recovery_traj_noise}). These patterns are consistent with the primary motivation of the model: recovering latent state-dependent structure when direct observations provide only a noisy view of the underlying dynamics.
\begin{figure}[H]
    \centering
    \includegraphics[width=1\linewidth]{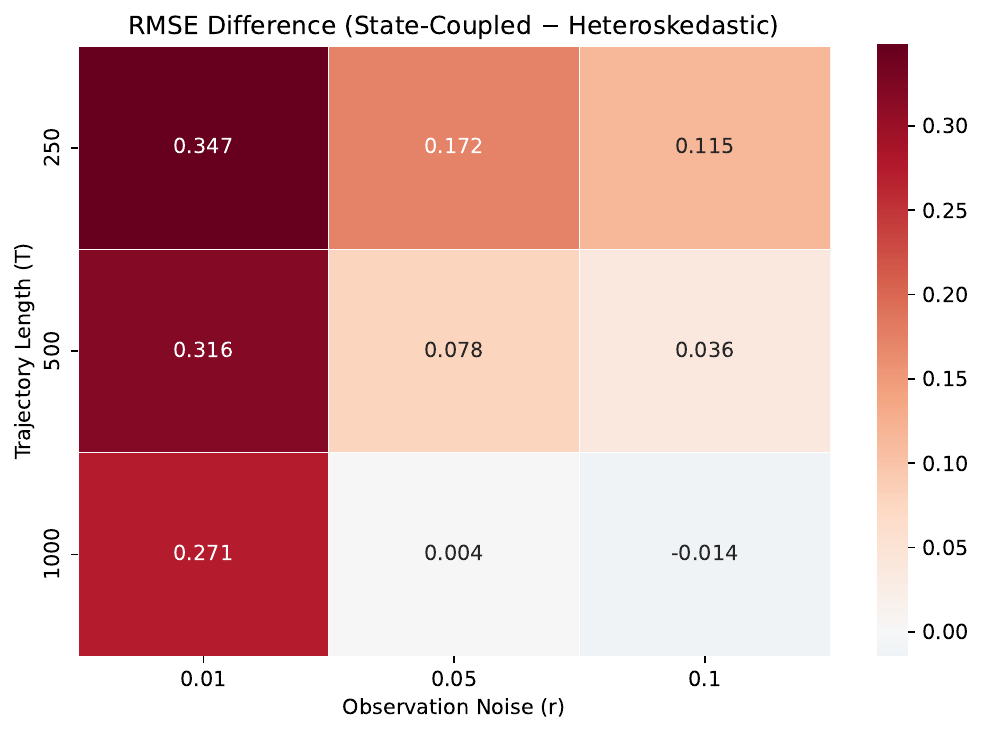}
    \caption{Difference in parameter recovery RMSE between the state-coupled framework and the observed-state heteroskedastic proxy across trajectory length ($T$) and observation noise ($r$). Values represent $\mathrm{RMSE}_{\text{Proxy}} - \mathrm{RMSE}_{\text{State-Coupled}}$, with positive values indicating lower recovery error for the proposed model. The largest improvements were observed for shorter trajectories and lower observation noise, while differences diminished as trajectory length increased and observations became noisier.
}
    \label{fig:recovery_traj_noise}
\end{figure}
\begin{figure} [H]
    \centering
    \includegraphics[width=1\linewidth]{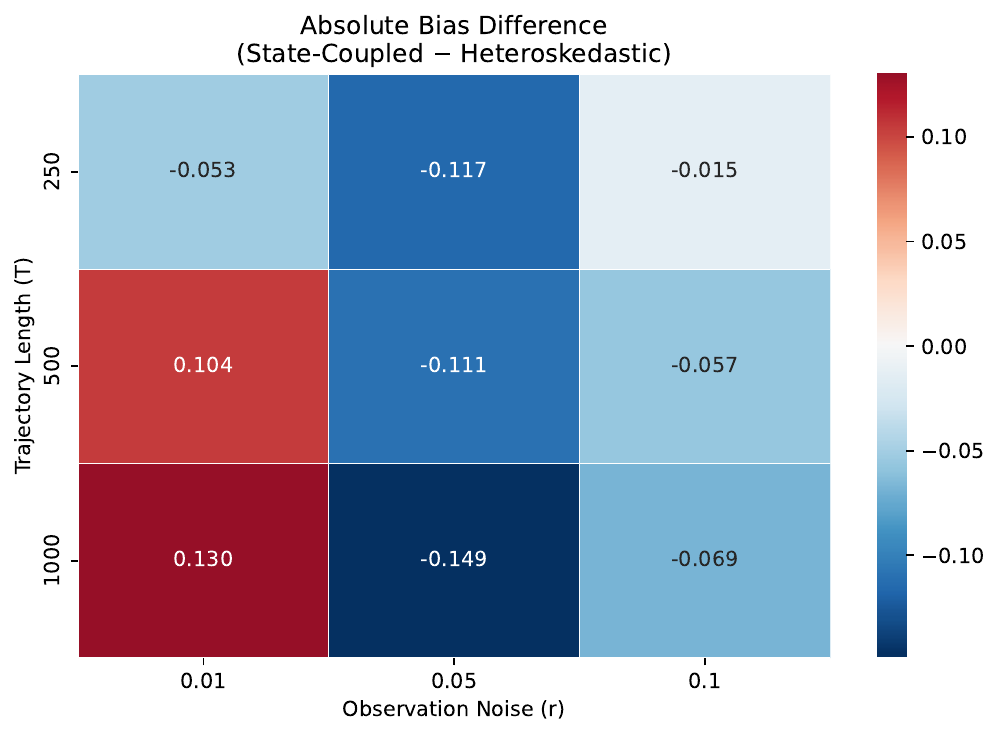}
    \caption{Difference in absolute recovery bias between the state-coupled framework and the observed-state heteroskedastic proxy across trajectory length ($T$) and observation noise ($r$). Values represent $|\mathrm{Bias}|_{\text{Proxy}} - |\mathrm{Bias}|_{\text{State-Coupled}}$, with positive values indicating lower absolute bias for the proposed model. Bias reduction was generally strongest for longer trajectories and moderate observation noise, although improvements were observed across most simulation conditions.
}
    \label{fig:trajbias}
\end{figure} 
\begin{table}[H]
\centering
\caption{Recovery performance across observation noise and trajectory length for the Heteroskedastic Proxy and State-Coupled models. Values show mean bias and RMSE with 95$\%$ confidence intervals.}
\label{tab:recovery_ci_r_T}
\begin{tabular}{llccc}
\toprule
Model & $r$ & $T$ & Bias (95$\%$ CI) & RMSE (95$\%$ CI) \\
\midrule
Heteroskedastic Proxy & $0.01$ & $250$  & $-0.149\,[-0.181,\,-0.116]$ & $0.379\,[0.349,\,0.411]$ \\
Heteroskedastic Proxy & $0.01$ & $500$  & $-0.103\,[-0.128,\,-0.078]$ & $0.288\,[0.266,\,0.311]$ \\
Heteroskedastic Proxy & $0.01$ & $1000$ & $-0.106\,[-0.123,\,-0.089]$ & $0.209\,[0.192,\,0.225]$ \\
State-Coupled         & $0.01$ & $250$  & $0.096\,[0.029,\,0.162]$    & $0.726\,[0.672,\,0.787]$ \\
State-Coupled         & $0.01$ & $500$  & $0.207\,[0.154,\,0.260]$    & $0.604\,[0.556,\,0.653]$ \\
State-Coupled         & $0.01$ & $1000$ & $0.236\,[0.197,\,0.274]$    & $0.480\,[0.447,\,0.513]$ \\
\bottomrule
\end{tabular}
\end{table}
\subsubsection*{Recovery Across Persistence and Observation Noise}
\begin{figure}[H]
    \centering
    \includegraphics[width=1\linewidth]{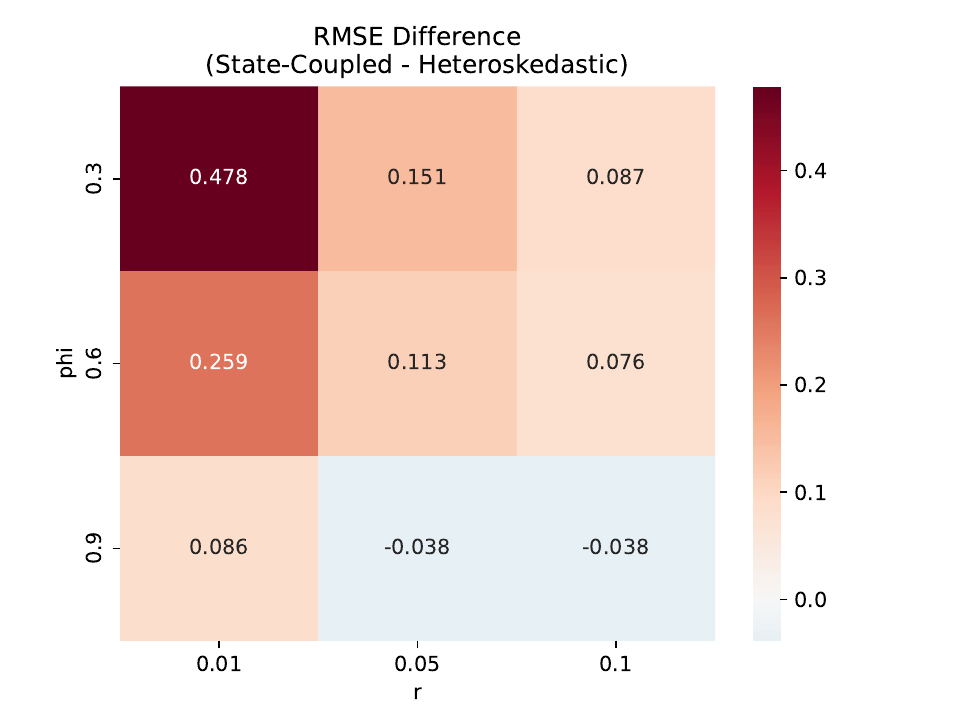}
    \caption{Difference in parameter recovery RMSE between the state-coupled framework and the observed-state heteroskedastic proxy across latent persistence ($\phi$) and observation noise ($r$).Values represent $\mathrm{RMSE}_{\text{Proxy}} - \mathrm{RMSE}_{\text{State-Coupled}}$, with positive values indicating lower recovery error for the proposed model. Performance advantages were greatest under low persistence and low-to-moderate observation noise, while differences diminished as persistence increased.
}
    \label{fig:latre}
\end{figure}
\begin{figure}[H]
    \centering
    \includegraphics[width=1\linewidth]{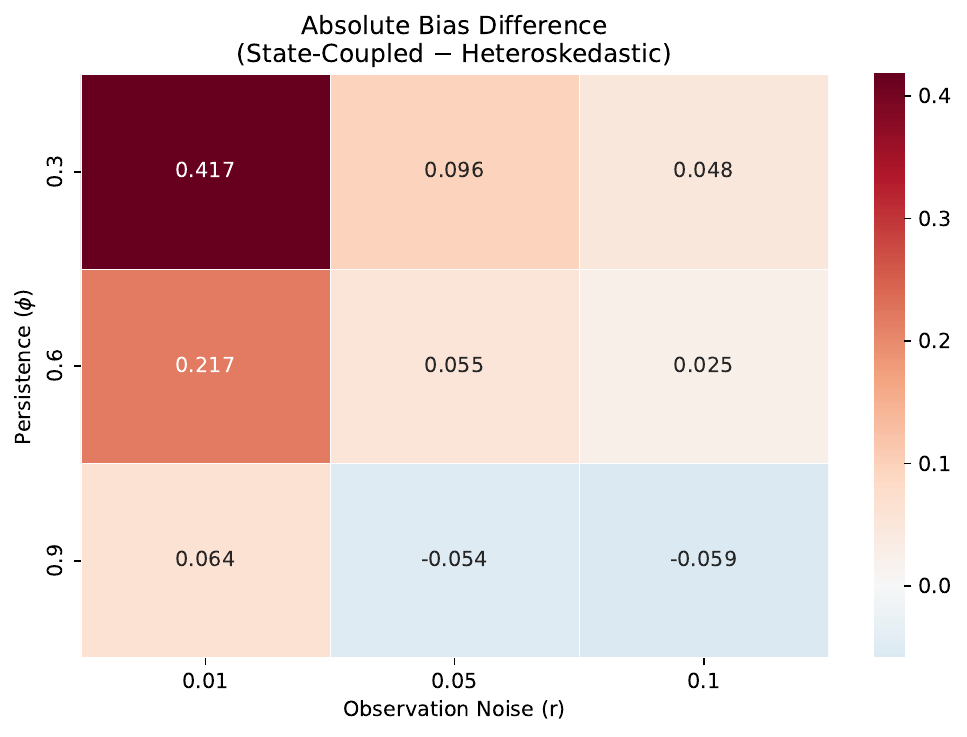}
    \caption{
Difference in absolute recovery bias between the state-coupled framework and the observed-state heteroskedastic proxy across latent persistence ($\phi$) and observation noise ($r$). Values represent
$|\mathrm{Bias}|_{\text{Proxy}} - |\mathrm{Bias}|_{\text{State-Coupled}}$, with positive values indicating lower absolute bias for the proposed model. Bias reduction was observed across nearly all simulation conditions, with the largest improvements occurring under low persistence and low observation noise.
}
    \label{fig:labi}
\end{figure}
The effects of observation noise depended strongly on the persistence of the underlying latent dynamics. When persistence was low $\phi=0.3$, both methods exhibited relatively large recovery errors, although the proposed state-coupled framework consistently produced estimates that were closer to the true coupling parameter. As observation noise increased, both approaches became more biased, although the heteroskedastic proxy generally continued to exhibit stronger attenuation of the coupling parameter (Figure ~\ref{fig:labi}).

As persistence increased, a different pattern emerged. Recovery improved substantially for both methods, suggesting that slowly evolving latent trajectories provide additional information about the relationship between state and variability. The improvement was particularly pronounced for the proposed state-coupled framework. Between $\phi=0.3$ and $\phi=0.9$, RMSE decreased from $0.852$ to $0.278$ at $r=0.01$, from $0.629$ to $0.250$ at $r=0.05$, and from $0.603$ to $0.296$ at $r=0.10$ (Table ~\ref{tab:phi_r_ci}).

The most favorable setting for the proposed model occurred under highly persistent dynamics $\phi=0.9$. In this regime, the proposed framework achieved both lower bias and lower RMSE than the heteroskedastic proxy in the moderate and high-noise settings. At $r=0.05$, the proposed model achieved an RMSE of $0.250$ compared with $0.288$ for the heteroskedastic proxy, while maintaining a smaller bias magnitude ($-0.085$ versus $-0.195$). A similar pattern was observed at $r=0.10$, where the proposed model achieved an RMSE of $0.296$ compared with $0.334$ for the heteroskedastic proxy (Figure ~\ref{fig:latre}).

These results suggest that the proposed framework benefits from situations in which the latent state remains informative over extended periods of time. When persistence is high, the model can more effectively leverage temporal structure to recover state-dependent variance dynamics, reducing both attenuation bias and overall recovery error. In contrast, when persistence is weak, the latent state becomes more difficult to reconstruct and the benefits of explicit latent-state modeling are diminished. Overall, the results indicate that the proposed framework is particularly well suited to systems characterized by both substantial observation noise and strong temporal persistence.
\begin{table}[ht]
\caption{Recovery performance by latent persistence ($\phi$) and observation noise ($r$). Values are mean bias and RMSE with 95$\%$ confidence intervals.}
\label{tab:phi_r_ci}
\begin{tabular}{lrrll}
\toprule
Model & $\phi$ & $r$ & Bias (95$\%$ CI) & RMSE (95$\%$ CI) \\
\midrule
heteroskedastic\_proxy & $0.3$ & $0.01$ & $-0.130\,[-0.163,\,-0.098]$ & $0.373\,[0.343,\,0.405]$ \\
heteroskedastic\_proxy & $0.3$ & $0.05$ & $-0.310\,[-0.344,\,-0.277]$ & $0.478\,[0.449,\,0.506]$ \\
heteroskedastic\_proxy & $0.3$ & $0.10$ & $-0.359\,[-0.393,\,-0.324]$ & $0.517\,[0.489,\,0.541]$ \\
heteroskedastic\_proxy & $0.6$ & $0.01$ & $-0.122\,[-0.148,\,-0.096]$ & $0.307\,[0.285,\,0.330]$ \\
heteroskedastic\_proxy & $0.6$ & $0.05$ & $-0.210\,[-0.243,\,-0.177]$ & $0.409\,[0.383,\,0.434]$ \\
heteroskedastic\_proxy & $0.6$ & $0.10$ & $-0.204\,[-0.239,\,-0.169]$ & $0.432\,[0.408,\,0.456]$ \\
heteroskedastic\_proxy & $0.9$ & $0.01$ & $-0.105\,[-0.119,\,-0.090]$ & $0.192\,[0.177,\,0.208]$ \\
heteroskedastic\_proxy & $0.9$ & $0.05$ & $-0.195\,[-0.215,\,-0.175]$ & $0.288\,[0.274,\,0.303]$ \\
heteroskedastic\_proxy & $0.9$ & $0.10$ & $-0.212\,[-0.236,\,-0.188]$ & $0.334\,[0.318,\,0.351]$ \\
state\_coupled & $0.3$ & $0.01$ & $0.333\,[0.260,\,0.406]$ & $0.852\,[0.802,\,0.908]$ \\
state\_coupled & $0.3$ & $0.05$ & $-0.135\,[-0.192,\,-0.078]$ & $0.629\,[0.583,\,0.678]$ \\
state\_coupled & $0.3$ & $0.10$ & $-0.325\,[-0.372,\,-0.277]$ & $0.603\,[0.563,\,0.644]$ \\
state\_coupled & $0.6$ & $0.01$ & $0.189\,[0.140,\,0.239]$ & $0.566\,[0.526,\,0.604]$ \\
state\_coupled & $0.6$ & $0.05$ & $-0.118\,[-0.165,\,-0.071]$ & $0.523\,[0.483,\,0.566]$ \\
state\_coupled & $0.6$ & $0.10$ & $-0.168\,[-0.213,\,-0.124]$ & $0.508\,[0.467,\,0.553]$ \\
state\_coupled & $0.9$ & $0.01$ & $0.016\,[-0.010,\,0.041]$ & $0.278\,[0.255,\,0.301]$ \\
state\_coupled & $0.9$ & $0.05$ & $-0.085\,[-0.107,\,-0.064]$ & $0.250\,[0.232,\,0.269]$ \\
state\_coupled & $0.9$ & $0.10$ & $-0.142\,[-0.166,\,-0.118]$ & $0.296\,[0.273,\,0.320]$ \\
\bottomrule
\end{tabular}
\end{table}
\subsubsection*{Recovery Across Coupling Strength and Persistence}
\begin{figure} [H]
    \centering
    \includegraphics[width=1\linewidth]{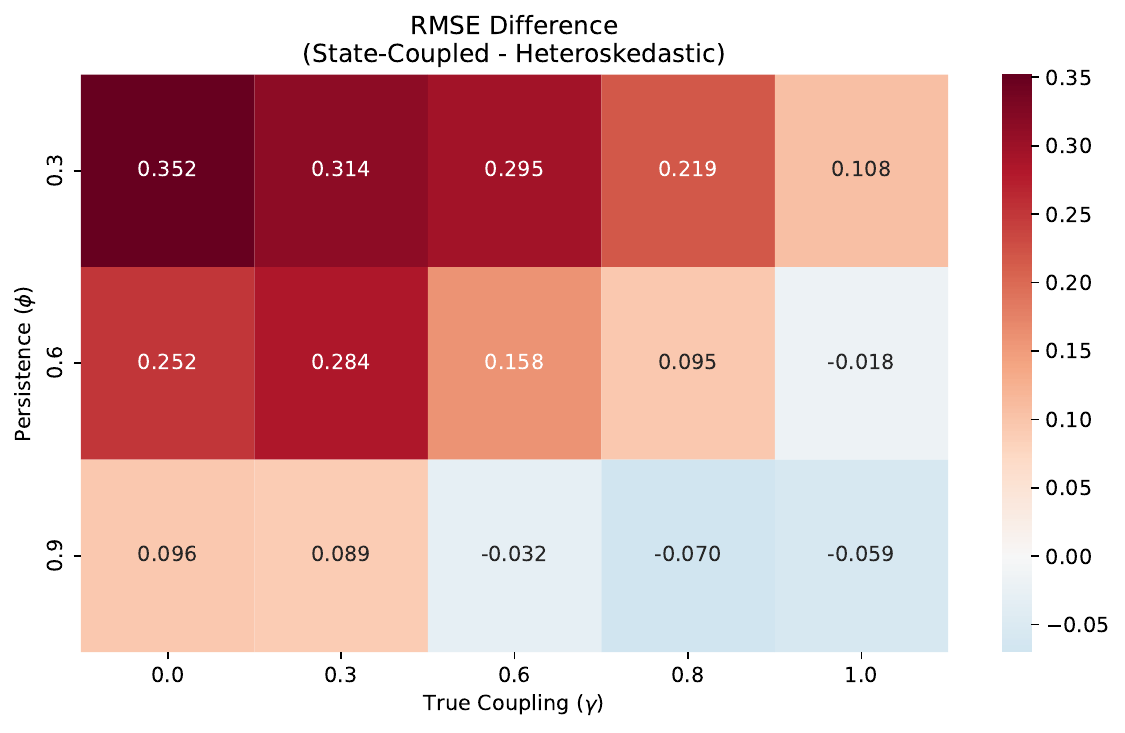}
    \caption{Difference in parameter recovery RMSE between the state-coupled framework and the observed-state heteroskedastic proxy across latent persistence ($\phi$) and coupling strength ($\gamma$). Values represent $\mathrm{RMSE}_{\text{State-Coupled}} - \mathrm{RMSE}_{\text{Proxy}}$; negative values indicate lower recovery error for the proposed model.}
    \label{fig:recper}
\end{figure}
\begin{figure}[H]
    \centering
    \includegraphics[width=1\linewidth]{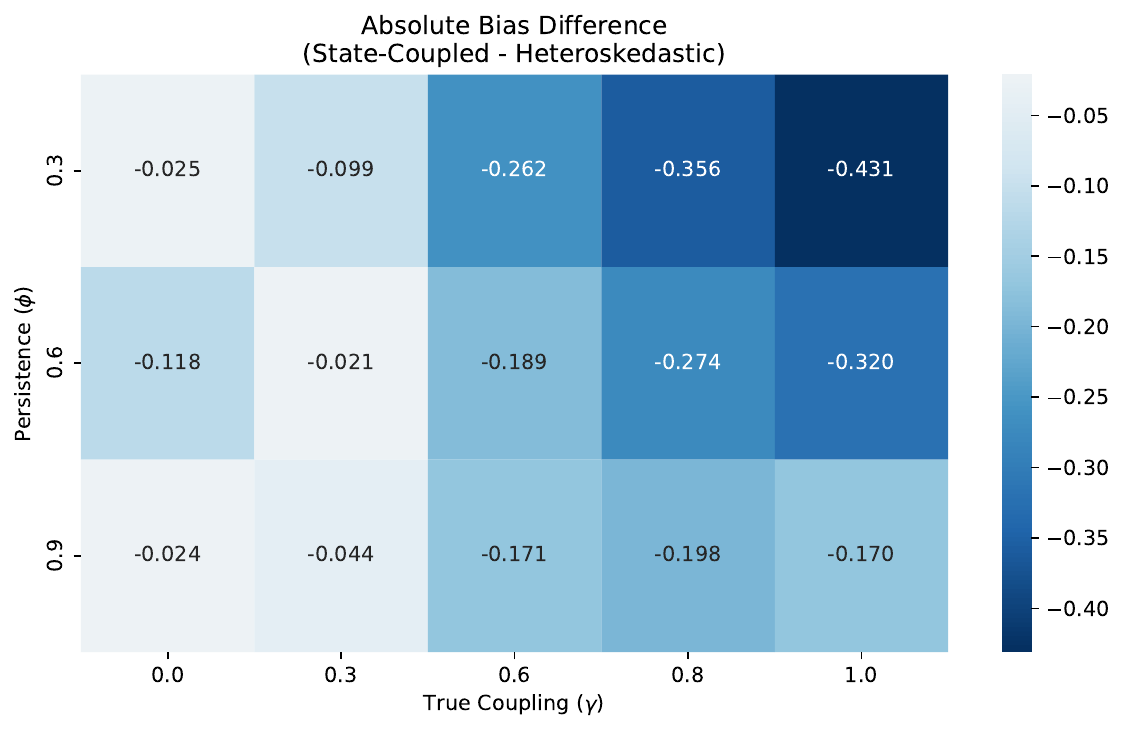}
    \caption{Difference in absolute recovery bias between the state-coupled framework and the observed-state heteroskedastic proxy across latent persistence ($\phi$) and coupling strength ($\gamma$). Values correspond to $|\mathrm{Bias}|_{\text{State-Coupled}} - |\mathrm{Bias}|_{\text{Proxy}}$, with negative values indicating lower absolute bias for the proposed model.
}
    \label{fig:absbicoup}
\end{figure}
The relationship between recovery performance, coupling strength, and latent persistence revealed several distinct patterns. For the heteroskedastic proxy, recovery bias increased steadily as the true coupling parameter increased. This trend was observed across all persistence levels, indicating that stronger state-dependent variance was progressively underestimated by the observed-state approach. For example, when $\phi=0.3$, bias increased from $-0.163$ at $\gamma=0.3$ to $-0.504$ at $\gamma=1.0$. Similar attenuation was observed for $\phi=0.6$ and $\phi=0.9$.

In contrast, the proposed state-coupled framework exhibited substantially greater stability across coupling strengths. Bias remained relatively small throughout the examined parameter range, with estimates generally remaining within approximately $\pm0.1$ of the true coupling value. For example, at $\phi=0.3$, bias varied only from $-0.064$ to $-0.073$ as coupling increased from $\gamma=0.3$ to $\gamma=1.0$. This suggests that the latent-state formulation is considerably less susceptible to attenuation as state-dependent variance becomes stronger.

The influence of persistence was equally pronounced. For both methods, recovery improved as $\phi$ increased, indicating that persistent latent trajectories provide additional information about the underlying variance structure. However, the improvement was especially notable for the proposed framework. At $\phi=0.9$, RMSE decreased substantially across all coupling strengths relative to the lower-persistence settings (Figure~\ref{fig:absbicoup}).

Most notably, under highly persistent dynamics $\phi=0.9$, the proposed framework achieved lower RMSE than the heteroskedastic proxy for moderate and strong coupling regimes. At $\gamma=0.6$, RMSE was $0.246$ for the proposed model compared with $0.278$ for the heteroskedastic proxy. Similar advantages were observed at $\gamma=0.8$ ($0.237$ versus $0.307$) and $\gamma=1.0$ ($0.317$ versus $0.376$) (Table ~\ref{tab:phi_gam_ci}). These gains occurred while maintaining substantially smaller bias magnitudes (Figure~\ref{fig:recper}) . 

The proposed framework benefits most from settings in which the latent state is both persistent and strongly coupled to variability. In such regimes, the model not only preserves recovery fidelity but can also outperform simpler observed-state approaches in overall estimation accuracy.
\begin{table}[ht]
\centering
\caption{Recovery performance by latent persistence ($\phi$) and coupling strength ($\gamma$). Values are reported as mean bias and RMSE with 95$\%$ confidence intervals.}
\label{tab:phi_gam_ci}
\small
\begin{tabular}{llcll}
\toprule
Model & $\phi$ & $\gamma$ & Bias (95$\%$ CI) & RMSE (95$\%$ CI) \\
\midrule
Heteroskedastic Proxy & $0.3$ & $0.0$ & $0.041\,[0.003,\,0.078]$ & $0.314\,[0.284,\,0.344]$ \\
Heteroskedastic Proxy & $0.3$ & $0.3$ & $-0.163\,[-0.197,\,-0.128]$ & $0.330\,[0.302,\,0.360]$ \\
Heteroskedastic Proxy & $0.3$ & $0.6$ & $-0.304\,[-0.340,\,-0.267]$ & $0.432\,[0.395,\,0.471]$ \\
Heteroskedastic Proxy & $0.3$ & $0.8$ & $-0.403\,[-0.444,\,-0.362]$ & $0.528\,[0.493,\,0.563]$ \\
Heteroskedastic Proxy & $0.3$ & $1.0$ & $-0.504\,[-0.547,\,-0.461]$ & $0.621\,[0.588,\,0.654]$ \\
\midrule
State-Coupled & $0.3$ & $0.0$ & $0.016\,[-0.064,\,0.096]$ & $0.666\,[0.600,\,0.731]$ \\
State-Coupled & $0.3$ & $0.3$ & $-0.064\,[-0.141,\,0.013]$ & $0.644\,[0.591,\,0.704]$ \\
State-Coupled & $0.3$ & $0.6$ & $-0.042\,[-0.129,\,0.046]$ & $0.727\,[0.652,\,0.811]$ \\
State-Coupled & $0.3$ & $0.8$ & $-0.047\,[-0.137,\,0.042]$ & $0.747\,[0.686,\,0.807]$ \\
State-Coupled & $0.3$ & $1.0$ & $-0.073\,[-0.161,\,0.014]$ & $0.729\,[0.671,\,0.787]$ \\
\bottomrule
\end{tabular}
\end{table}
\subsubsection*{Detection Performance Across Simulation Settings}
\begin{figure}[H]
    \centering
    \includegraphics[width=1\linewidth]{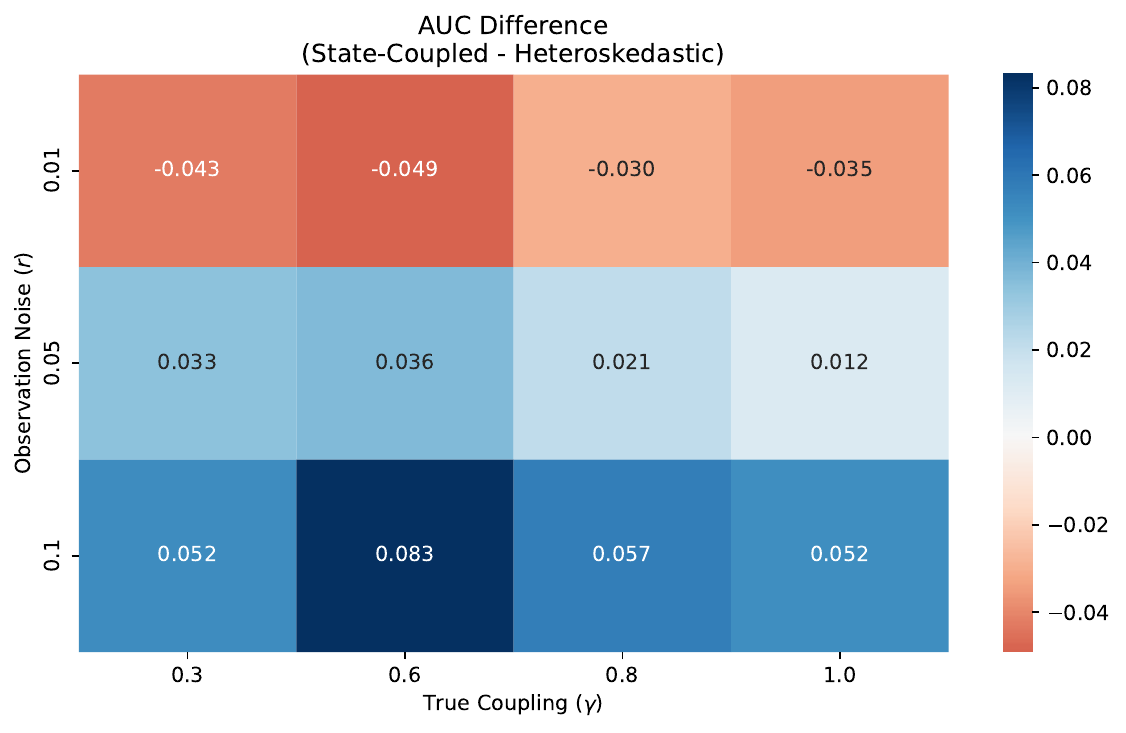}
    \caption{Difference in detection performance AUC between the proposed state-coupled framework and the heteroskedastic proxy across observation noise levels r and coupling strengths $\gamma$. Positive values indicate superior detection performance for the proposed method.}
    \label{fig:obsauc}
\end{figure}
\begin{figure}[H]
    \centering
    \includegraphics[width=1\linewidth]{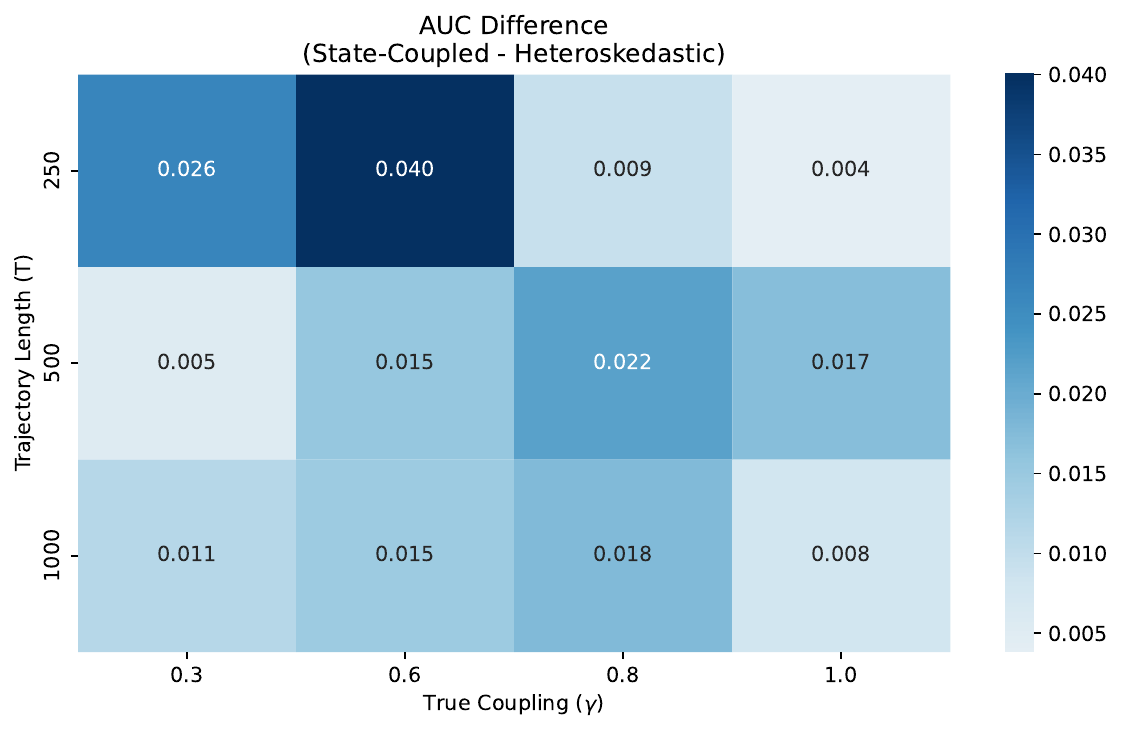}
    \caption{Difference in detection performance AUC between the proposed state-coupled framework and the heteroskedastic proxy across trajectory length T and coupling strengths $\gamma$. Positive values indicate superior detection performance for the proposed method.}
    \label{fig:traauc}
\end{figure}
\begin{figure}[H]
    \centering
    \includegraphics[width=1\linewidth]{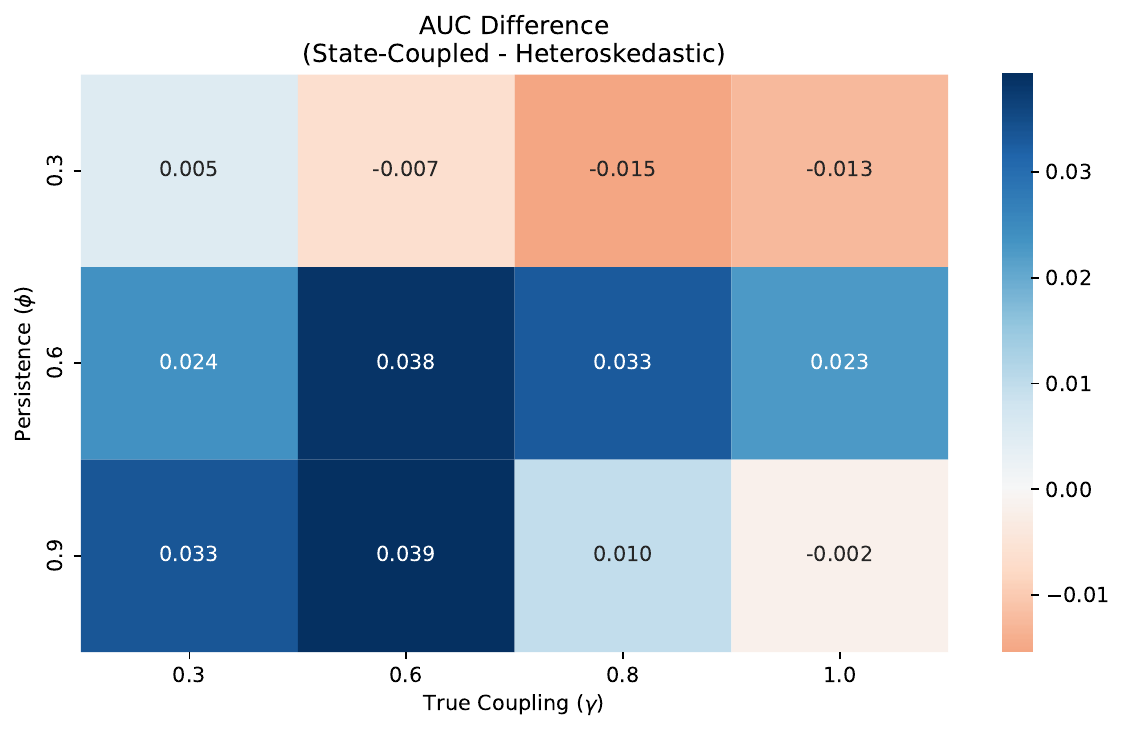}
    \caption{Difference in detection performance AUC between the proposed state-coupled framework and the heteroskedastic proxy across latent persistence  $\phi$ and coupling strengths $\gamma$. Positive values indicate superior detection performance for the proposed method.}
    \label{fig:perauc}
\end{figure}
Detection performance was generally high across all methods and simulation conditions, although the relative advantage of the proposed framework depended on the underlying characteristics of the latent process.

Trajectory length had only a modest effect on the relative performance of the competing methods. Across all combinations of coupling strength and trajectory length, the proposed state-coupled framework achieved slightly higher AUC values than the heteroskedastic proxy, with improvements ranging from $0.004$ to $0.040$ (Figure~\ref{fig:traauc}). These gains were consistent across all trajectory lengths examined, suggesting that additional observations improved detection performance for both methods without substantially altering their relative ranking.

In contrast, observation noise had a much larger impact on comparative performance. Under low-noise conditions $r=0.01$, the heteroskedastic proxy achieved slightly higher AUC values across all coupling strengths, with differences ranging from $0.030$ to $0.049$ in favor of the proxy. However, this pattern reversed as noise increased. At moderate noise levels $r=0.05$, the proposed framework achieved higher AUC values across all coupling strengths, while under high-noise conditions $r=0.10$ the advantage became more pronounced, with improvements ranging from $0.052$ to $0.083$ (Figure~\ref{fig:obsauc}). These findings suggest that the primary detection advantage of the proposed framework arises when observations provide an imperfect representation of the latent process, allowing latent-state reconstruction to recover information obscured by measurement noise.

Detection performance also depended on latent persistence. Under weak persistence $\phi=0.3$, differences between methods were small and generally favored the heteroskedastic proxy, with AUC differences ranging from $-0.015$ to $0.005$. As persistence increased, the proposed framework became increasingly advantageous. At $\phi=0.6$, the proposed model consistently outperformed the heteroskedastic proxy across all coupling strengths, with AUC improvements ranging from $0.023$ to $0.038$. Similar improvements were observed at $\phi=0.9$, although the advantage diminished slightly at the strongest coupling levels (Figure~\ref{fig:perauc}).

Taken together, these results suggest that observation noise and latent persistence are the primary determinants of detection advantage. The proposed framework provided its greatest improvements when observations were noisy and the latent dynamics exhibited moderate to strong temporal persistence. In contrast, when observations closely reflected the latent state or when persistence was weak, simpler observed-state approaches achieved comparable or slightly better detection performance. Notably, the proposed model never exhibited large detection disadvantages, while providing clear gains in several practically relevant regimes characterized by noisy observations and persistent latent dynamics.
\subsubsection*{Comparison With Alternative Volatility Models}
\begin{figure}[H]
    \centering
    \includegraphics[width=1\linewidth]{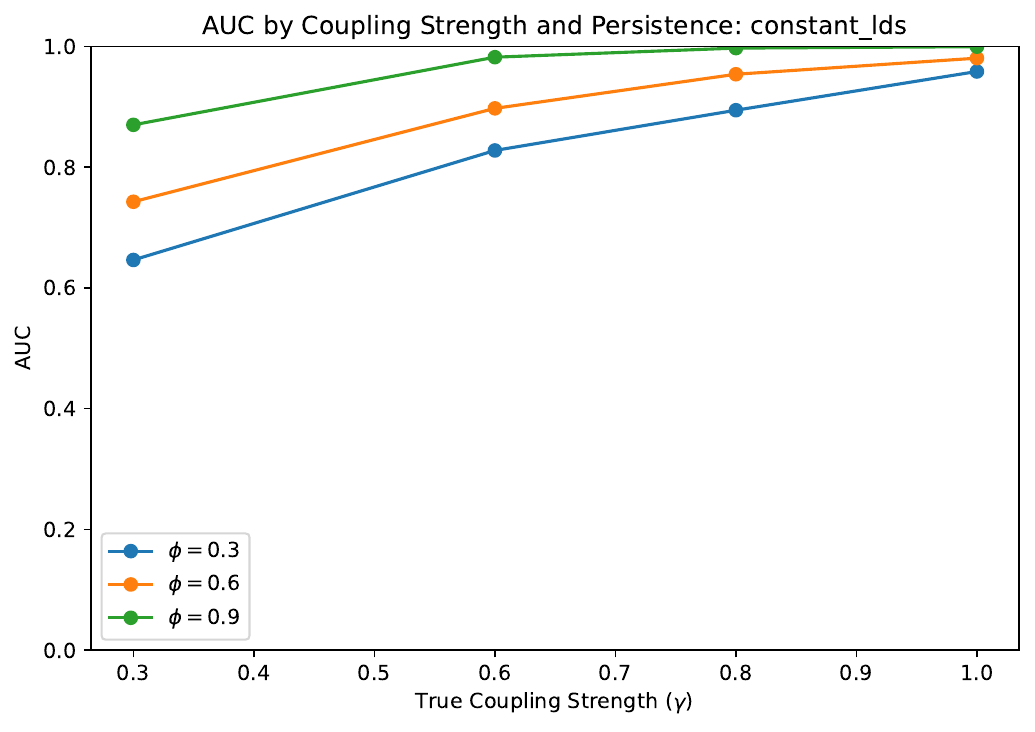}
    \caption{Detection performance AUC of the constant-variance latent dynamical system Constant LDS across coupling strengths $\gamma$ and persistence levels $\phi$. Although the model does not explicitly represent state-dependent variance, detection performance increases with coupling strength because stronger coupling produces departures from the constant-variance assumptions of the model.}
    \label{fig:conlds}
\end{figure}
\begin{figure}[H]
    \centering
    \includegraphics[width=1\linewidth]{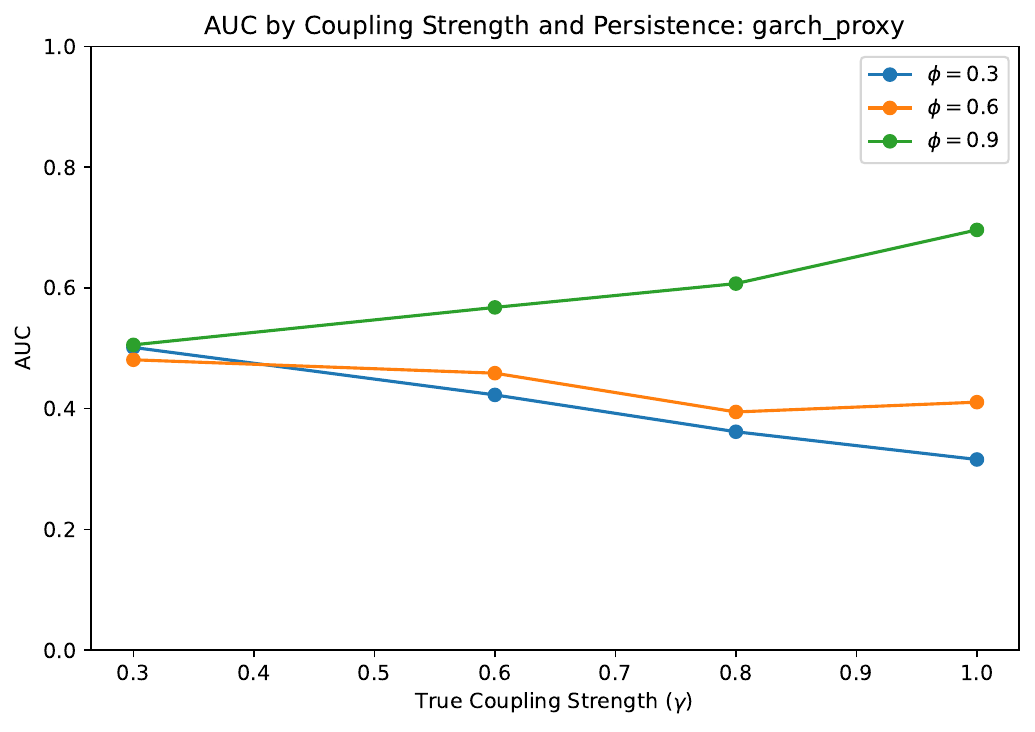}
    \caption{AUC of the GARCH proxy across coupling strengths $\gamma$ and persistence levels $\phi$. Performance remains near chance across most simulation settings, indicating that volatility models based primarily on variance clustering are poorly aligned with the state-dependent variance mechanism generating the data.}
    \label{fig:garchauc}
\end{figure}
\begin{figure}[H]
    \centering
    \includegraphics[width=1\linewidth]{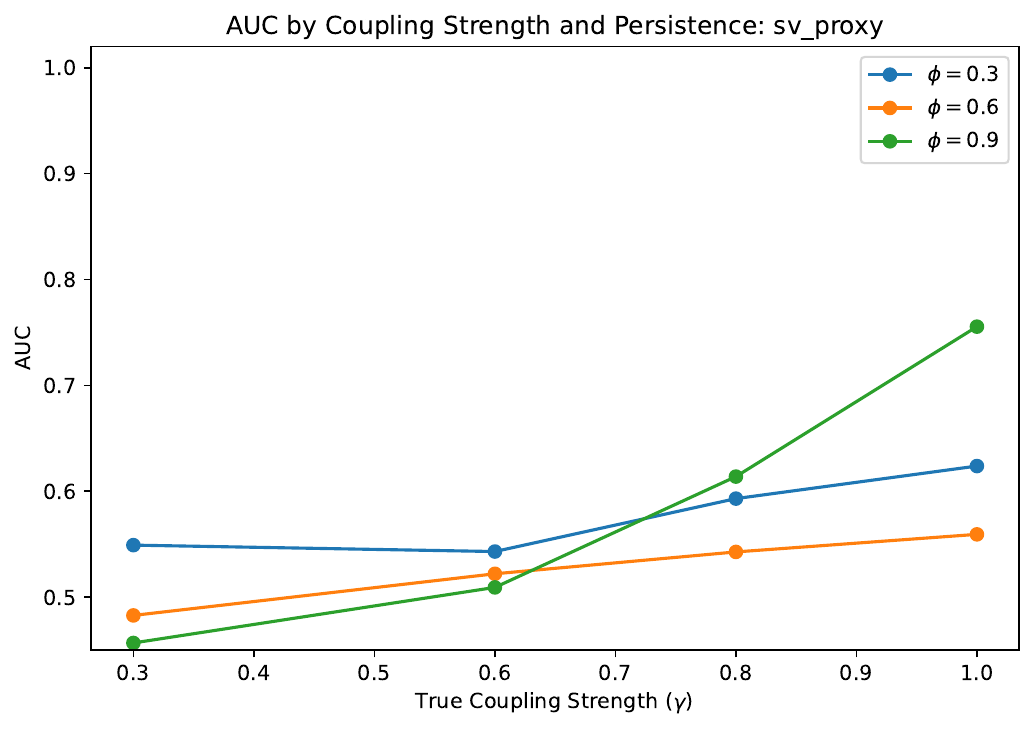}
    \caption{AUC of the SV proxy across coupling strengths $\gamma$ and persistence levels $\phi$. Detection performance improves as coupling strength increases but remains below that of models incorporating explicit state information. These results suggest that generic volatility models capture only limited information about variance structures that depend on latent-state position.}
    \label{fig:svproxy}
\end{figure}
To evaluate whether state-coupled variance could be detected using existing volatility modeling approaches, we compared the proposed framework against three alternative methods: Constant LDS, a GARCH-based proxy, and a SV proxy. Because these models do not estimate an explicit state-coupling parameter, comparisons focused on discrimination performance as measured by AUC.

The Constant LDS exhibited surprisingly strong detection performance despite lacking an explicit coupling mechanism. The AUC values generally exceeded $0.85$ and approached $1.0$ for larger coupling strengths. This behavior is expected because strong state-coupled variability induces deviations from the assumptions of a constant-variance model, allowing misspecification itself to act as an indirect indicator of coupling (Figure~\ref{fig:conlds}). However, the Constant LDS provides no estimate of coupling magnitude and therefore cannot be used for recovery of the underlying state-dependent relationship.

The GARCH proxy performed substantially worse than both the proposed framework and the heteroskedastic proxy. In most simulation settings, AUC values remained close to chance levels, typically ranging between $0.45$ and $0.55$ (Figure~\ref{fig:garchauc}). These results suggest that classical volatility clustering models are poorly suited for detecting the specific form of state-dependent variance considered here, where variability depends on latent-state position rather than purely on recent volatility history.

The stochastic volatility proxy exhibited modestly better performance than the GARCH approach but remained substantially below the proposed framework. Detection performance generally improved as coupling strength increased, with AUC values reaching approximately $0.60--0.70$ in the strongest coupling regimes (Figure~\ref{fig:svproxy}). Nevertheless, performance remained considerably lower than both the proposed state-coupled model and the heteroskedastic proxy.

These results indicate that conventional volatility models capture only limited information about state-dependent variance structure. While model misspecification in the Constant LDS can serve as a useful detector of non-constant variance, neither GARCH nor stochastic volatility approaches appear well aligned with the underlying mechanism generating the simulated data. The strongest overall performance was achieved by methods that explicitly incorporate state information, with the proposed framework providing the additional advantage of direct coupling recovery.
\section*{Discussion}

The present study introduced a state-coupled stochastic volatility framework for modeling latent state-dependent variability in partially observed dynamical systems and evaluated its performance across a broad range of simulated conditions. Several consistent patterns emerged. Most notably, the proposed framework substantially reduced recovery bias relative to observed-state heteroskedastic approaches, particularly when state-volatility coupling was strong. Although this improvement was often accompanied by increased estimation variability, the reduction in systematic error remained remarkably consistent across simulation settings, suggesting that explicit latent-state reconstruction provides important information that is unavailable when variability is modeled solely from observed measurements.

A central finding was the persistent attenuation bias exhibited by the observed-state heteroskedastic proxy. As observation noise increased, estimates of the coupling parameter became increasingly biased toward zero. This behavior is expected because measurement noise obscures the true relationship between latent-state position and process variability. When variability is modeled directly from noisy observations, the resulting estimates conflate measurement error with the underlying variance structure, producing systematic underestimation of coupling strength. By reconstructing latent trajectories before estimating the variance mapping, the proposed framework partially separates these sources of uncertainty and therefore produces estimates that remain substantially closer to the true coupling parameter.

An important pattern across the simulation study was a tradeoff between systematic error and estimation variability. The proposed framework consistently reduced attenuation bias, often by a substantial margin, but these improvements did not always translate into lower RMSE. In several low-noise conditions, the observed-state heteroskedastic proxy achieved lower RMSE despite exhibiting greater bias, reflecting the fact that direct estimation from observations can produce more stable estimates when measurements already provide a close approximation of the latent process. Conversely, the proposed framework incurred additional uncertainty associated with latent-state reconstruction while producing estimates that remained substantially closer to the true coupling parameter. This pattern suggests that the principal advantage of the state-coupled framework is improved recovery fidelity rather than uniformly lower estimation error. As latent persistence increased and additional temporal information became available, however, the variance cost of latent-state reconstruction diminished, allowing the bias reduction to translate more directly into improvements in overall recovery performance.

Recovery performance was also strongly influenced by the temporal structure of the latent process. Across multiple simulation regimes, increasing persistence improved both bias and RMSE for all methods, with particularly large gains for the proposed framework. This pattern suggests that temporal persistence provides additional information that facilitates latent-state reconstruction. When latent dynamics evolve gradually, information accumulates across observations, allowing the filtering and smoothing procedures to recover the underlying trajectory more accurately. In contrast, weak persistence reduces the amount of information available about latent-state evolution, making recovery of state-dependent variance considerably more difficult. The particularly strong performance observed under high-persistence conditions indicates that the framework is most effective when latent states remain informative over extended periods of time.

The benchmark comparisons further clarify the nature of the information captured by the proposed model. The Constant LDS achieved surprisingly strong detection performance despite lacking any explicit state-dependent variance mechanism. This result suggests that sufficiently strong coupling leaves detectable signatures in the dynamics themselves, allowing model misspecification to serve as an indirect indicator of latent variance structure. However, because the Constant LDS cannot estimate a coupling parameter, it provides no information about the magnitude or form of the underlying relationship. Similarly, the weak performance of the GARCH and stochastic volatility benchmarks indicates that state-dependent variability differs fundamentally from generic forms of time-varying volatility. Variance processes driven by residual history or independent latent volatility trajectories capture only limited aspects of the phenomenon considered here. The strongest performance was consistently observed among models that incorporated information about system state, emphasizing the importance of state-dependent representations of variability.

The interaction between observation noise and latent-state reconstruction was particularly informative. Under low-noise conditions, the observed-state heteroskedastic proxy often achieved lower RMSE despite exhibiting larger bias. In these settings, observations already provide a relatively accurate representation of the latent process, reducing the need for explicit latent-state inference. As observation noise increased, however, the advantages of the proposed framework became increasingly apparent. Detection performance improved relative to competing approaches, attenuation bias was substantially reduced, and overall recovery accuracy became increasingly competitive. These findings suggest that latent-state reconstruction is most valuable precisely in the settings for which it was designed: systems in which observations provide only an imperfect representation of the underlying dynamics.

Importantly, the simulation regimes in which the proposed framework performed best are not unusual corner cases. Many physiological, behavioral, and neural systems are characterized by substantial measurement noise and persistent latent dynamics. Examples include neural population activity, electrophysiological recordings, autonomic regulation, motor adaptation, and treatment-response trajectories. Consequently, the conditions under which the framework demonstrated the greatest advantages correspond closely to situations frequently encountered in empirical research. This observation suggests that explicitly modeling state-dependent variability may provide useful information about system dynamics beyond that available from mean-state trajectories alone.

Several limitations should be acknowledged. The present study relied entirely on simulated data generated under the assumed model structure. Although the simulations were designed to span a broad range of conditions, the results demonstrate methodological feasibility rather than empirical validity. The findings therefore should not be interpreted as evidence that state-coupled volatility exists in any specific biological, behavioral, or physiological system. Real-world datasets may contain additional complexities, including model misspecification, nonstationarity, measurement artifacts, and nonlinear dynamics that were not considered here. Future work should evaluate the framework using empirical datasets, examine robustness under model misspecification, and investigate more flexible variance mappings capable of representing richer forms of state-dependent variability.

Overall, the results demonstrate that state-coupled volatility can be recovered under partial observation when latent-state structure is explicitly modeled. Across a wide range of simulation conditions, the proposed framework consistently reduced attenuation bias and frequently improved detection performance, particularly in noisy environments characterized by persistent latent dynamics. These findings establish a methodological foundation for investigating whether structured variability contributes meaningful information about system dynamics beyond that contained in mean-state trajectories alone.

\section*{Conclusion}

This work introduced a state-coupled stochastic volatility framework for studying latent state-dependent variability in partially observed dynamical systems. The proposed model extends conventional latent state-space formulations by allowing process variance to vary as a function of displacement from a latent equilibrium, providing an interpretable representation of structured stochasticity. To estimate this relationship under partial observation, we developed a particle expectation-maximization procedure combining bootstrap particle filtering and backward trajectory smoothing.

Across a large simulation benchmark, the proposed framework consistently reduced recovery bias relative to observed-state heteroskedastic alternatives and demonstrated its greatest advantages under conditions of substantial observation noise and persistent latent dynamics. Recovery accuracy improved with increasing temporal persistence, while detection performance remained competitive across a broad range of simulation settings and was often superior in noisy observational regimes. These findings suggest that latent-state reconstruction can preserve information about state-dependent variability that may be difficult to recover from observed measurements alone.

Taken together, the results demonstrate that state-coupled volatility can be identified and estimated under partial observation when the latent-state structure is explicitly modeled. More generally, the framework provides a methodological foundation for investigating whether structured variability contributes useful information about system dynamics beyond that contained in mean-state trajectories alone. Because the present study is based entirely on simulated data, the findings should be interpreted as evidence of methodological feasibility rather than validation of state-coupled volatility in empirical systems. Future work should evaluate the framework on real-world datasets, assess robustness under model misspecification, and extend the approach to more flexible variance mappings and nonlinear latent dynamics.
\section*{Acknowledgments}

The author acknowledges the University of California San Diego for providing an academic environment that supported the development of this work. The views expressed in this manuscript are solely those of the author.
\section*{Code Availability}

To support reproducibility and future methodological development, the proposed state-coupled stochastic volatility framework has been implemented as an open-source Python package, \texttt{varaware}. The package includes implementations for latent-state simulation, particle filtering, backward trajectory smoothing, particle expectation-maximization estimation, and visualization of model performance and operating characteristics.

The source code is publicly available through GitHub:

\begin{center}
\texttt{https://github.com/imabec/varaware}
\end{center}

The repository is intended to facilitate application, evaluation, extension, and benchmarking of the proposed framework in other latent dynamical system settings.

The simulation scripts used to generate the results reported in this paper are not currently included in the public repository but are available from the author upon request.

\bibliographystyle{plain}
\bibliography{ref}
\end{document}